%% file: main.tex
\documentclass[11pt]{article}
\usepackage[margin=1in]{geometry}
\pdfoutput=1 

\usepackage[utf8]{inputenc}
\usepackage[round]{natbib}
\usepackage{multirow}
\usepackage[pdftex]{graphicx}
\usepackage{booktabs}
\usepackage{array}
\usepackage{zref-abspage}

\usepackage{algorithm}
\usepackage{algorithmic}
\renewcommand{\algorithmicrequire}{\textbf{Input:}} 
\renewcommand{\algorithmicensure}{\textbf{Output:}}

\usepackage{mathtools}
\usepackage{lipsum}
\usepackage{enumitem}

\usepackage[linktocpage,colorlinks,linkcolor=blue,anchorcolor=blue,citecolor=blue,urlcolor=blue,pagebackref]{hyperref}

\newcommand{\kbala}[1]{{\color{magenta}\bf[KB: #1]}}

\input{Text/abhimacros}
\usepackage{comment}
\let\hat\widehat
\let\tilde\widetilde

\usepackage[toc,page]{appendix}
\usepackage[perpage]{footmisc}

\usepackage{times}

\usepackage{subfigure}
\usepackage[export]{adjustbox}
\usepackage{xcolor}

\usepackage{verbatim}

\usepackage{amsmath,amsfonts,bm, amsthm}
\usepackage{amssymb}

\usepackage{macros}
\usepackage{tikz}
\usetikzlibrary{positioning, arrows}
\usepackage{wrapfig}
\usepackage{selectp,todonotes,mdframed}

\newcommand{\Yifan}[1]{{\textcolor{magenta}{[Yifan: #1]}}}

\renewcommand*{\backref}[1]{\ifx#1\relax \else Page #1 \fi}
\renewcommand*{\backrefalt}[4]{%
  \ifcase #1 \footnotesize{(Not cited.)}%
  \or        \footnotesize{(Cited on page~#2.)}%
  \else      \footnotesize{(Cited on pages~#2.)}%
  \fi
}

\begin{document}
\title{\Large \bf Stochastic Optimization Algorithms for \\Instrumental Variable Regression with Streaming Data}
\author{
    Xuxing Chen$^\sharp$ \thanks{Department of Mathematics, University of California, Davis. Email: \href{mailto:xuxchen@ucdavis.edu}{xuxchen@ucdavis.edu}.    
    }
    \and
    Abhishek Roy$^\sharp$ \thanks{ Halıcıoğlu Data Science Institute, University of California, San Diego. Email: \href{mailto:a2roy@ucsd.edu}{a2roy@ucsd.edu}.    
    }
    \and
	Yifan Hu\thanks{College of Management of Technology, EPFL, Department of Computer Science, ETH Zurich, Switzerland. Email:
	\href{mailto:yifan.hu@epfl.ch}{yifan.hu@epfl.ch}. YH is supported by NCCR Automation of Swiss National Science Foundation.
	}
	\and
	Krishnakumar Balasubramanian\thanks{Department of Statistics, University of California, Davis. Email: \href{mailto:kbala@ucdavis.edu}{kbala@ucdavis.edu}. KB is supported by NSF grant DMS-2053918.
 }
}
\date{\today}
\maketitle
\def\thefootnote{$\sharp$}\footnotetext{XC and AR contributed equally to this work.}\def\thefootnote{\arabic{footnote}}

\begin{abstract}
    \input{Text/0_abstract}
\end{abstract}
\input{Text/1_introduction}
\input{Text/2_two_sample}
\input{Text/Online2SLS1}

\input{Text/5_numerical_studies}

\input{Text/6_conclusion}


\medskip
\bibliographystyle{abbrvnat}
\bibliography{ref}
\begin{appendix}
\input{Text/7_appendix}

\end{appendix}
\end{document}

%% file: Text/abhimacros.tex
\def\g{{\gamma}}
\def\t{{\theta}}

\def\betone{{\beta_{t+1}}}

\def\betsqr{{\beta_{t}^2}}
\def\bet{{\beta_{t}}}
\def\atone{{\alpha_{t+1}}}
\def\atonesqr{{\alpha_{t+1}^2}}
\def\atsqr{{\alpha_{t}^2}}
\def\at{{\alpha_{t}}}
\def\gtone{{\gamma_{t+1}}}
\def\sigmar{{\vartheta}}
\def\gt{{\gamma_{t}}}
\def\ttone{{\theta_{t+1}}}
\def\tt{{\theta_{t}}}
\def\deltone{{\delta_{t+1}}}
\def\delt{{\delta_{t}}}
\def\epstt{{\epsilon_{2,t}}}
\newcommand{\vbrac}[1]{\left|#1\right|}
\newcommand{\lrbrac}[1]{\left(#1\right)}

\def\hQt{{\hat{Q}_{t}}}
\newcommand{\expec}[2]{\mathbb{E}_{#2}\left[ #1 \right] }
\newcommand\numberthis{\addtocounter{equation}{1}\tag{\theequation}}  

\newcommand{\normtt}[1]{\lVert #1\rVert^2}

%% file: Text/0_abstract.tex
We develop and analyze algorithms for instrumental variable regression by viewing the problem as a conditional stochastic optimization problem. In the context of least-squares instrumental variable regression, our algorithms neither require matrix inversions nor mini-batches and provides a fully online approach for performing instrumental variable regression with streaming data. When the true model is linear, we derive rates of convergence in expectation, that are of order $\mathcal{O}(\log T/T)$ and $\mathcal{O}(1/T^{1-\iota})$ for any $\iota>0$, respectively under the availability of two-sample and one-sample oracles, respectively, where $T$ is the number of iterations. Importantly, under the availability of the two-sample oracle, our procedure avoids explicitly modeling and estimating the relationship between confounder and the instrumental variables, demonstrating the benefit of the proposed approach over recent works based on reformulating the problem as minimax optimization problems. Numerical experiments are provided to corroborate the theoretical results. 




%% file: Text/1_introduction.tex
\section{Introduction}\label{sec:intro}

Instrumental variable analysis is widely used in fields like econometrics, health care, social science, and online advertisement to estimate the causal effect of a random variable, $X$, on an outcome variable, $Y$, when an unobservable confounder influences both. By identifying an instrumental variable correlated with the variable $X$ but unrelated to the confounders, researchers can isolate the exogenous variation in $X$ and estimate a causal relationship between $X$ and $Y$. In the context of regression, Instrumental Variable Regression (IVaR) addresses endogeneity issues when an independent variable is correlated with the error term in the regression model, leveraging an instrument variable~$Z$ such that~$Y$ is independent of~$X|Z$. In this paper, we focus on the following statistical model: 
\begin{align}\label{eq:IVMainModel}
Y=g_{\theta^*}(X)+\epsilon_1\quad\text{with}\quad X=h_{\gamma^*}(Z)+\epsilon_2
\end{align}
where $X \in \mathbb{R}^{d_x}$ and $\epsilon_1$ are correlated and $\epsilon_2$ is a centered unobserved noise (independent of $Z\in\mathbb{R}^{d_z}$), leading to confounding in the model between $X$ and $Y \in \mathbb{R}$. Here $\epsilon_1$ and $\epsilon_2$ are dependent, and $\theta^*$ and $\gamma^*$ are true parameters for the respective function $g$ and $h$. Our goal is to design efficient algorithms that recovers $\theta^*$ from the data. 


Traditionally, IVaR algorithms are based on two-stage estimation procedures, where we first regress $Z$ and $X$ to obtain an estimator $\hat X$, and then regress $\hat X$ and $Y$, with the essence that $\hat X$ is independent of $Y$, and thus eliminating the aforementioned endogeneity of the unknown confounder.  
A vast literature has devoted to understanding the two-stage approaches~\citep{hall2005nonparametric, darolles2011nonparametric, hartford2017deep}, with the parametric two-stage least-squares (2SLS) procedure being the most canonical one~\citep{angrist1995two}. The main drawback of this approach is that the second-stage regression problem is affected by the estimation error from the regression problem corresponding to first stage. In fact,~\cite{angrist2009mostly} call the first stage regression as ``forbidden regression'', due to the concerns in estimating a nuisance parameter. 

Considering the squared loss function,
\cite{muandet2020dual} formulate the IVaR problem as a conditional stochastic optimization problem~\citep{hu2020biased}: 
\begin{align}
\label{problem:IV_parameterized}
\min_{g\in\mathcal{G}} F(g):=\EE_{Z} \EE_{Y\mid Z}[ (Y-\EE_{X\mid Z}[g(X)])^2].
\end{align}
However,~\cite{muandet2020dual} did not solve problem \eqref{problem:IV_parameterized} efficiently, and resort to reformulating~\eqref{problem:IV_parameterized} further as a minimax optimization problem. Indeed, they mention explicitly in their work that ``\emph{it
remains cumbersome to solve \eqref{problem:IV_parameterized} directly because of the inner expectation}''. Then, they leverage the Fenchel conjugate of the squared loss, 
leading to a minimax optimization with maximization over a continuous functional space. Following \cite{dai2017learning},  \cite{muandet2020dual} propose to use reproducing kernel Hilbert space (RKHS) to handle the maximization over continuous functional space. See also~\cite{lewis2018adversarial,bennett2019deep,dikkala2020minimax,liao2020provably,bennett2023minimax} for similar minimax approaches. The issue with such an approach is that approximating the dual variable via maximization over continuous functional space inevitably introduces approximation error. Hence, although there is no explicit nuisance parameter estimation step like in the two-stage approach, there is an implicit one, which makes the minimax approach less appealing as an alternate to the two-stage procedures.



In this work, contrary to the claim made in~\cite{muandet2020dual} that problem~\eqref{problem:IV_parameterized} is cumbersome to solve, we design and analyze efficient streaming algorithms to directly solve the conditional stochastic optimization problem in~\eqref{problem:IV_parameterized}. 
%
Direct application of methods from \cite{hu2020biased} for solving  \eqref{problem:IV_parameterized} is possible, yet  their approach utilizes nested sampling, i.e., for each sample of $Z$, \cite{hu2020biased} generate a batch of samples of $X$ from $\PP(Z|X)$, to reduce the bias in estimating the composition of non-linear loss function with conditional expectations. Thus their methods are not suitable for the streaming setting that we are interested in. Considering~\eqref{problem:IV_parameterized}, we first parameterize the function class $\mathcal{G}:=\{g(\theta;X) \mid \theta \in\RR^{d_\theta}\}$. Now, defining $F(g)\coloneqq F(\theta)$,  we observe that the gradient $\nabla F(\theta)$ admits the following form
\begin{align}
\label{eq:gradient_IV}
\nabla F(\theta) = \EE_{Z}[ (\EE_{X\mid Z}[g(\theta;X)] - \EE_{Y\mid Z} [Y])\nabla_\theta \EE_{X\mid Z}[g(\theta;X)]],
\end{align}
which implies that one does not need the nested sampling technique to reduce the bias. However, the presence of product of two conditional expectations $\EE_{X\mid Z}[g(\theta;X)]$ still causes significant challenges in developing stochastic estimators of the above gradient in the streaming setting. In this work, we overcome this challenge and develop two algorithms that are applicable to the streaming data setting avoiding the need for generating batches of samples of $X$ from $\PP(Z|X)$.\\


\noindent\textbf{Contributions.} 
We make the following contributions in this work.
\begin{itemize}[leftmargin=0.15in]
    \item \textbf{Two-sample oracles:} Our first algorithm leverages the observation that if we have access to a two-sample oracle that outputs \emph{two} samples $X$ and $X'$ that are independent conditioned on the instrument $Z$, we can immediately construct an unbiased stochastic gradient estimator of the gradient in~\eqref{eq:gradient_IV}. Based on this crucial observation, we propose the \emph{Two-Sample One-stage Stochastic Gradient IVaR} (TOSG-IVaR) method (Algorithm~\ref{alg:two_sample_SGD}) that avoids explicitly having to estimate or model the relationship between $Z$ and $X$ thereby overcoming the ``forbidden regression'' problem.. Under standard statistical model assumptions, for the case when $g$ is a linear model, we establish rates of convergence of order $\mathcal{O}(\log T/T)$ for the proposed method, where $T$  is the overall number of iterations; see Theorem~\ref{thm:convergence_linear_two_sample}.
    \item \textbf{One-sample oracles:} In the case when we do not have the aforementioned two-sample oracle, we estimate the stochastic gradient in~\eqref{eq:gradient_IV} by using the streaming data to estimate one of the conditional expectations, and the corresponding prediction to estimate the other, resulting in the \emph{One-Sample Two-stage Stochastic Gradient IVaR} (OTSG-IVaR) method (Algorithm~\ref{alg:one_sample_onlineIV}). Assuming further that the $X$ depends linearly on the instrument $Z$, we establish a rate of convergence of order $\mathcal{O}(1/T^{1-\iota})$, for any $\iota>0$; see Theorem~\ref{th:mainthetaconv}.  
\end{itemize}

\subsection{Literature Review}
\noindent\textbf{IVaR analysis.} Instrumental variable analysis has a long history, starting from the early works by~\cite{wright1928tariff} and \cite{reiersol1945confluence}. Several works considered the aforementioned two-stage procedure for IVaR; a summary could be found in the work by~\cite{angrist2009mostly}. Nonparametric approaches based on wavelets, splines, reproducing kernels and deep neural networks 
could be found, for example, in the works by~\cite{hartford2017deep,singh2019kernel,bennett2019deep, muandet2020dual,mastouri2021proximal,xu2021learning,zhu2022causal,peixoto2024nonparametric}. Another popular approach for IVaR is via Generalized Method of Moments (GMM); see, for example,~\cite{chen2012estimation,bennett2019deep,dikkala2020minimax} for an overview. Such approaches essentially reformulate the problem as a minimax problem and hence suffer from the aforementioned ``forbidden regression'' problem.\\


\noindent \textbf{Identifiability conditions for IVaR.} Several works in the literature have also focused on establishing the identifiability conditions for IVaR in the parametric and the nonparametric setting. Regardless of the procedure used, they are invariably based on certain source conditions motivated by the inverse problems literature (see, for example,~\cite{carrasco2007linear,chen2011rate, bennett2023minimax}) or the related problem of completeness conditions, which posits that the conditional expectation operator is one-to-one~\citep{babii2017completeness,liao2020provably}. Semi-parametric identifiability is also considered recently in the work of~\cite{cui2023semiparametric}. Our focus in this work is not focused on the identifiability; for the formulation~\eqref{problem:IV_parameterized} that we consider,~\cite{muandet2020dual} provide necessary conditions for identifiability that we adopt.\\

\noindent\textbf{Stochastic optimization with nested expectations.} Recently, much attention in the stochastic optimization literature has focused on optimizing a nested composition of $T$ expectation functions. Sample average approximation algorithms in this context are considered in the works of~\cite{ermoliev2013sample} and~\cite{hu2020sample}. Optimal iterative stochastic optimization algorithms for the case of $T=2$ were by derived by~\cite{ghadimi2020single}. For the general $T\geq 1$ case,~\cite{wang2017stochastic} provided sub-optimal rates, whereas \cite{balasubramanian2022stochastic} derived optimal rates; see also~\cite{zhang2021multilevel} and~\cite{chen2021solving} for related works under stronger assumptions, and~\cite{ruszczynski2021stochastic} for similar asymptotic results. While the above works required certain independence assumptions regarding the randomness across the different compositions,~\cite{hu2020biased,hu2024contextual} studied the case of $T=2$ where the the randomness are generically dependent. They termed this problem setting as conditional stochastic optimization, which is the framework that the IVaR problem in~\eqref{problem:IV_parameterized} falls in. Compared to prior works, for e.g.,~\cite{ghadimi2020single} and~\cite{balasubramanian2022stochastic}, in order to handle the dependency between the levels,~\cite{hu2020biased} require mini-batches in each iteration, making their algorithm not immediately applicable to the purely streaming setting. In this work, we show that despite the problem~\eqref{problem:IV_parameterized} being a conditional stochastic optimization problem, mini-batches are not required due the additional favorable quadratic structure available in IVaR.\\

\noindent \textbf{Streaming IVaR.}~\cite{venkatraman2016online,della2023online} analyzed streaming versions of 2SLS in the online\footnote{Their notion of online is from the literature on \emph{online learning}~\citep{shalev2012online}.} and adversarial settings. Focusing on linear models, \cite{venkatraman2016online} provide preliminary asymptotic analysis assuming access to efficient \emph{no-regret learners}, while \cite{della2023online} provide regret bounds under the strong assumption that the instrument is almost surely bounded. Furthermore, our algorithms have significantly improved per-iteration and memory complexity compared to~\cite{della2023online}; see Sections~\ref{alg:pseudoools} and~\ref{sec:periter} for details. \cite{chen2023sgmm} developed stochastic optimization algorithms for the GMM formulation and provide asymptotic analysis. Their algorithm requires access to an offline dataset for initialization and is hence not fully online. The above works (i) do not focus on avoiding the forbidden regression problem and (ii) do not view IVaR via the \emph{conditional stochastic optimization} lens, like we do.










%% file: Text/2_two_sample.tex
\section{Two-sample One-stage Stochastic Gradient Method for IVaR}
\label{sec:two_sample}\vspace{-0.05in}
Recall that our goal is to solve the objective function given in~\eqref{problem:IV_parameterized}.
By \citet[Theorem 4]{muandet2020dual}, the optimal solution of \eqref{problem:IV_parameterized} gives the true underlying causal relationship under the following assumption.
\begin{assumption}
\label{assumption:idenfitication}
(Identifiability Assumption)
\begin{itemize}[noitemsep,leftmargin=0.12in]
    \item The conditional distribution $\PP_{Z\mid X}$ is continuous in $Z$ for any value of $X$.
    \item The function class $\mathcal{G}:=\{g(\theta;X) \mid \theta \in\RR^{d_\theta}\}$ is correctly specified, i.e., it includes the true underlying relationship between $X$ and $Y$.
\end{itemize}   
\end{assumption}
Notice that both assumptions are standard in the IVaR literature \citep{newey2003instrumental,chen2012estimation,muandet2020dual}, and makes the objective in \eqref{problem:IV_parameterized} is the meaningful for IVaR. However, \cite{muandet2020dual}  resort to reformulating the objective function in \eqref{problem:IV_parameterized} as a minimax optimization problem as described in Section~\ref{sec:intro}. While their original motivation was to avoid two-state estimation procedure and avoid the ``forbidden regression'', their minimax reformulation ends up having to solve a  complicated approximation of the original objective resulting in having to characterize the approximation error which is non-trivial.\\ 

\textbf{Algorithm and Analysis.} Our aim in this work is to directly solve the original problem in~\eqref{problem:IV_parameterized}, leveraging the structure provided by the quadratic loss. Given the gradient formulation in \eqref{eq:gradient_IV}, 
a natural way to build unbiased gradient estimator is to generate $X$ and $X^\prime$, two independent samples of $X$ from the conditional distributions $\PP_{X\mid Z}$, for a given realization of $Z$ and generate one sample of $Y$ from the conditional distribution $\PP_{Y\mid X}$. Then, an unbiased gradient estimator is 
\begin{equation}\label{eq: grad_estimator}
    v(\theta) = (g(\theta;X) - Y)\nabla_\theta g(\theta;X^\prime).
\end{equation}
This could be plugged into the standard stochasic gradient descent algorithm, which give us the Two-sample Stochastic Gradient Method for IVR (TSG-IVaR) method illustrated in Algorithm~\ref{alg:two_sample_SGD}. 
In particular, the algorithm never requires estimating (or modeling) the relationship between $X$ and $Z$ as needed in the two-stage procedure~\citep{angrist2009mostly} and the minimax formulation based procedures~\citep{muandet2020dual,lewis2018adversarial,bennett2019deep,dikkala2020minimax,liao2020provably,bennett2023minimax}. Furthermore, this viewpoint not only provides a novel algorithm for performing IV regression, but also provides a novel data collection mechanism for the practical implementation of IVaR. In addition, such a two-sample gradient method is not very restrictive when the instrumental variable $Z$ takes value in a discrete set. In this case, to implement the two-sample oracle, it is enough simply pick two sets of samples $(X,Y,Z)$ and $(X',Y',Z)$ for which $Z$ has repeated observations (which is possible when $Z$ is a discrete random variable) from a pre-collected dataset. To demonstrate the convergence rate of Algorithm \ref{alg:two_sample_SGD}, we first consider the case when $g$ is a linear function, i.e., $g(\theta;X)=X^\top \theta$. We make the following assumptions.




\begin{algorithm}[t]
	\caption{Two-sample One-stage Stochastic Gradient-IVaR (\texttt{TOSG-IVaR})}
	\label{alg:two_sample_SGD}
	\begin{algorithmic}[1]
		\REQUIRE  $\sharp$ of iterations~$T$, stepsizes~$\{\alpha_t\}_{t=1}^{T}$, initial iterate~$\theta_1$.
            \FOR{$t=1$ to~$T$ \do} 
            \STATE Sample $Z_t$, sample independently $X_t$ and $X_t^\prime$  from $\PP_{X\mid Z_t}$, and sample $Y_t$ from $\PP_{Y \mid X_t}$. 
            \STATE Update $\theta_t$ 
            $$
            \theta_{t+1} = \theta_t - \alpha_{t+1} (g(\theta_t;X_t) - Y_t)\nabla_\theta g(\theta_t;X_t^\prime).
            $$
            \ENDFOR
		\ENSURE~$\theta_T$.
	\end{algorithmic}
\end{algorithm}

\begin{assumption}\label{aspt: scvx}
    Suppose there exists $\mu>0$ such that $            \EE_Z\Big[\EE_{X\mid Z}[X]\cdot \EE_{X\mid Z}[X]^\top\Big] \succeq \mu I.$
\end{assumption}

\begin{assumption}\label{aspt: var_general}
    Let $(\vartheta_1, \vartheta_2, \vartheta_3, \vartheta_4) \in \mathbb{R}_+^4$. For any $Z$, $X'$ and $X$ i.i.d. generated from $\PP_{Z\mid X}$ , and $Y$ generated from $\PP_{Y\mid X}$, and for constants $C_{x},C_y,C_{xx},C_{yx}>0$,  we have
    \begin{align}
        &\EE\Big[\norm{X' X^\top - \EE_{X|Z}[X]\EE_{X|Z}[X]^\top}^2\Big]\leq C_{x}d_x^{\vartheta_1}, \label{ineq: px}\\  
        &\EE\Big[\norm{YX - \EE_{Y|Z}[Y]\EE_{X|Z}[X]}^2\Big]\leq C_y d_x^{\vartheta_2}, \label{ineq: py}\\ 
        &\EE\Big[\norm{\EE_{X\mid Z}[X]\cdot \EE_{X\mid Z}[X]^\top -\EE_Z\Big[\EE_{X\mid Z}[X]\cdot \EE_{X\mid Z}[X]^\top\Big]}^2\Big] \leq C_{xx}d_z^{\vartheta_3}, \label{ineq: pxx}\\ 
        &\EE\Big[\norm{\EE_{Y\mid Z}[Y]\cdot \EE_{X\mid Z}[X] -\EE_Z\Big[\EE_{Y\mid Z}[Y]\cdot \EE_{X\mid Z}[X]\Big]}^2\Big] \leq C_{yx}d_z^{\vartheta_4}, \label{ineq: pyx}
    \end{align}
where $\|\cdot\|$ denotes the Euclidean norm and operator norm for a vector and matrix respectively.
\end{assumption}
The above assumptions are mild moment assumptions required on the involved random variables. The following result demonstrates that Assumptions \ref{aspt: scvx} and \ref{aspt: var_general} are naturally satisfied under even under non-linear modeling assumption on~\eqref{eq:IVMainModel}. We defer its proof to Section \ref{appendix:lemma1_proof}. 

\begin{lemma}\label{lem:bdd_var_linear}
    Suppose there exist $\theta_*\in \RR^{d_x},\ \gamma_*\in\RR^{d_z\times d_x},$ a non-linear map $\phi: \RR^{d_x}\rightarrow \RR^{d_x}$, and a positive semi-definite matrix $\Sigma\in \RR^{d_z\times d_z}$ such that 
\begin{align}
    &\EE_Z\Big[\phi(\gamma_*^\top Z)\cdot \phi(\gamma_*^\top Z)^\top\Big]\succeq \mu I,\ \EE[\norm{\phi(\gamma_*^\top Z)}^2] = \cO(d_x), \notag \\
    &Z \sim \cN(0, \Sigma),\ X = \phi(\gamma_*^\top Z) + \epsilon_2,\ Y = \theta_*^\top X + \epsilon_1,\ \epsilon_2\sim \cN(0, \sigma_{\epsilon_2}^2I_{d_x}),\ \epsilon_1 \sim \cN(0, \sigma_{\epsilon_1}^2), \label{eq: special_nonlinear_setup}
\end{align}
where $\epsilon_1, \epsilon_2$ are independent of $Z$ and
\begin{equation}\label{ineq: special_nonlinear_setup_ineq}
\EE\left[\epsilon_1^2\norm{\epsilon_2}^2\right]\leq \sigma_{\epsilon_1,\epsilon_2}^2d_x,\ \EE\left[\norm{\phi(\gamma_*^\top Z)\cdot \phi(\gamma_*^\top Z)^\top - \EE[\phi(\gamma_*^\top Z)\cdot \phi(\gamma_*^\top Z)^\top]}^2\right]\leq Cd_z,
\end{equation}
then Assumptions \ref{aspt: scvx} and \ref{aspt: var_general} hold with $\vartheta_1=\vartheta_2=2$ and $\vartheta_3=\vartheta_4=1$. if $\phi$ is an identity map, then the conditions involving $\phi$ become $
    \gamma_*^\top \Sigma\gamma_*\succeq \mu I,\ \tr(\gamma_*^\top\Sigma\gamma_*) = \cO(d_x),\ \EE\left[\norm{ZZ^\top - \Sigma}^2\right]\leq Cd_z.$

\end{lemma}

\begin{assumption}\label{aspt: independence}
    The tuple $(Z_t, X_t, X'_t, Y_t)$ is independent and identically distributed, across $t$. 
\end{assumption}
The above assumption is standard in the stochastic approximation, statistics and econometrics literature. It could be further relaxed to Markovian-type dependency assumptions, following techniques in the works of~\cite{duchi2012ergodic,sun2018markov,even2023stochastic, roy2022constrained}; we leave a detailed examination of the Markovian streaming setup as future work. Under the above assumptions, we have the following result demonstrating the last-iterate global convergence of Algorithm \ref{alg:two_sample_SGD}.

\begin{theorem}
\label{thm:convergence_linear_two_sample}
    Suppose Assumptions \ref{aspt: scvx}, \ref{aspt: var_general}, and \ref{aspt: independence} hold. In Algorithm \ref{alg:two_sample_SGD}, defining $\sigma_1^2 \coloneqq 2C_{x}d_x^{\vartheta_1} + 2C_{xx}d_z^{\vartheta_3}$ and $\sigma_2^2 \coloneqq C_y d_x^{\vartheta_2} + C_{yx}d_z^{\vartheta_4}$, set $\alpha_t \equiv \alpha = \frac{\log T}{\mu T}\leq \frac{\mu}{\mu^2 + 3\sigma_1^2}$. Then, we have 
    \begin{align*}
        \EE\big[\norm{\theta_T - \theta_*}^2\big]\leq \frac{\EE\big[\norm{\theta_0 - \theta_*}^2\big]}{T} + \frac{3\norm{\theta_*}^2(\sigma_1^2 + \sigma_2^2)\log T}{\mu^2T}.
    \end{align*}
\end{theorem}

\noindent\textbf{Proof techniques.} In the analysis of Theorem \ref{thm:convergence_linear_two_sample}, the following decomposition (see \eqref{eq: theta_decompose_1} for the derivation) plays a crucial role:
    \begin{align*}
        &\theta_{t+1} - \theta_* = A_t + \alpha_{t+1} B_t,\\
        &A_t = \theta_t - \alpha_{t+1}\EE_Z\Big[\EE_{X\mid Z}[X]\cdot \EE_{X\mid Z}[X]^\top\Big] \theta_t + \alpha_{t+1}\EE_Z\Big[\EE_{Y\mid Z}[Y]\cdot \EE_{X\mid Z}[X] \Big] - \theta_*, \notag\\
        &B_t = - \Big(X_t'X_t^\top - \EE_Z\Big[\EE_{X\mid Z}[X]\cdot \EE_{X\mid Z}[X]^\top\Big] \Big)\theta_t + \Big(Y_tX_t' -  \EE_Z\Big[\EE_{Y\mid Z}[Y]\cdot \EE_{X\mid Z}[X] \Big]\Big),
    \end{align*}
    where $A_t$ corresponds to deterministic component, and $B_t$ corresponds to the stochastic component arising due to the use of stochastic gradients. Standard assumptions on the variance of the stochastic gradient made in the  stochastic optimization literature include the uniformly bounded variance assumption~\citep{lan2020first} and the expected smoothness condition \citep{khaled2020better}. In the IVaR setup, such standard assumptions do not hold as $\theta_t$ potentially can be unbounded and thus the gradient estimator can be unbounded. Hence, we establish our results under natural statistical assumptions arising in the context of the IVaR problem, which form the main novelty in our analysis. 
Furthermore, compared to \cite{muandet2020dual}, notice that we use two samples of $X$ from the conditional distribution $\PP_{X|Z}$ and achieve an $\tilde \cO(1/T)$ last iterate convergence rate to the global optimal solution, which is the true underlying causal relationship under Assumption \ref{assumption:idenfitication}. In comparison, \cite{muandet2020dual} only provide asymptotic convergence result to the optimal solution of an approximation problem.\\ 

\textbf{Additional discussion.} It is interesting to explore other losses beyond squared loss (for example to handle classification setting~\citep{centorrino2021nonparametric}), potentially using the Multilevel Monte Carlo (MLMC) based stochastic gradient estimators. While~\cite{hu2021bias}, develops such algorithms, the main challenge is about how to avoid mini-batches required in their work leveraging the problem structure in instrumental variable analysis. Furthermore, in the case when $g(\theta;X)$ is parametrized by a non-linear models, for instance, a neural network, we provide local convergence guarantees  under additional stronger conditions made typically in the stochastic optimization literature. 
\begin{assumption}\label{aspt: nonlinear}
Let the following assumptions hold:
\begin{itemize}[noitemsep, leftmargin=0.21in]
    \item Function $F(\theta)$ is $\ell$-smooth.
    \item The iterates $\{\theta_t\}_{t=1}^{T+1}$ generated by Algorithm \ref{alg:two_sample_SGD} are in a compact set $A$.
    \item The random objects $X|Z$ and $Y|Z$ have bounded variance for any $Z$, i.e., there exist $\sigma > 0$ such that
    \begin{align*}
        \EE\left[\norm{X - \EE\left[X\mid Z\right]}^2\mid Z\right]\leq \sigma^2,\ \EE\left[\norm{Y - \EE\left[Y\mid Z\right]}^2\mid Z\right]\leq \sigma^2.
    \end{align*}
\end{itemize}
\end{assumption}



\begin{proposition}\label{thm:convergence_noncvx_two_sample}
    Suppose Assumptions \ref{assumption:idenfitication}, \ref{aspt: independence}, and \ref{aspt: nonlinear} hold. Choosing $\alpha_t\equiv \alpha = \mathcal{O}\left(\frac{1}{\sqrt{T}}\right)$,  for Algorithm \ref{alg:two_sample_SGD} we have
    \begin{align*}
        \min_{1\leq t\leq T}\EE\left[\norm{\nabla F(\theta_t)}^2\right] = \mathcal{O}\left(\frac{1}{\sqrt{T}}\right).
    \end{align*}
\end{proposition}
The proof of the proposition is immediate. Note that under Assumption~\ref{aspt: nonlinear}, we can deduce that the unbiased gradient estimator $v(\theta) = (g(\theta; X) - Y)\nabla_{\theta}g(\theta; X')$ has a bounded variance since
\begin{align*}
    \text{Var}(v(\theta)) = &\text{Var}(g(\theta; X) - Y)\text{Var}(\nabla_{\theta}g(\theta; X')) \\
    &+\text{Var}(g(\theta; X) - Y)\left(\EE\left[\nabla_{\theta}g(\theta; X')\right]\right)^2 + \text{Var}(\nabla_{\theta}g(\theta;X'))\left(\EE\left[g(\theta; X) - Y\right]\right)^2 \leq \sigma_{v}^2,
\end{align*}
where the variance and expectation are taken conditioning on $Z$ and $\theta$, and $\sigma_v>0$ is a constant that only depends on $\sigma$, function $g$ and the compact set $A$ in Assumption \ref{aspt: nonlinear}. Then one can directly follow the analysis of non-convex stochastic optimization (see, for example, \citet[Theorem 2.1]{ghadimi2013stochastic}) to obtain Proposition \ref{thm:convergence_noncvx_two_sample}. Relaxing the Assumption \ref{aspt: nonlinear} (typically made in the stochastic optimization literature) with more natural assumptions on the statistical model and obtaining a result as in Theorem~\ref{thm:convergence_linear_two_sample} for the non-convex setting is left as future work.

%% file: Text/Online2SLS1.tex
\section{One-sample Two-stage Stochastic Gradient Method for IVaR}\label{sec:one_sample}

We now examine designing streaming IVaR algorithm with access to the classical one-sample oracle, i.e.,  we observe a streaming set of samples $(X_t,Y_t,Z_t)$ at each time point $t$. Note that in this case, using the same $X_t$ (instead of $X'_t$) in~\eqref{eq: grad_estimator} makes the stochastic gradient estimator biased.\\

\noindent \textbf{Intuition.} 
Consider the case of linear models, i.e., $Y=\theta_*^\top X +\epsilon_1$ with $X=\gamma_*^\top Z+\epsilon_2,$
where $\t_*\in\mathbb{R}^{d_x\times 1}$, and $\g_*\in\mathbb{R}^{d_z\times d_x}$, as also considered in Lemma~\ref{lem:bdd_var_linear}. Recall the true gradient in~\eqref{eq:gradient_IV} and the stochastic gradient estimator of Algorithm~\ref{alg:two_sample_SGD} in~\eqref{eq: grad_estimator}. Since we no longer have $X_t'$, we replace the term $X_t'$ with the predicted mean of $X_t$ given $Z_t$. Suppose that $\g_*$ is known. We specifically replace $\nabla_{\tt} g(\tt;X_t')=X_t'$ by $\expec{X_t}{\mid Z_t}=\g_*^\top Z_t$. 
In such a case, indeed we have an unbiased gradient estimator: 
\begin{align*}
    &\expec{ \g_*^\top Z_t(X_t^\top\tt-Y_t)}{t}
       = \expec{ \expec{X_t}{\mid Z_t}(\expec{X_t}{\mid Z_t}^\top\tt-\expec{Y_t}{\mid Z_t})}{t} \\ =&\expec{\g_*^\top Z_tZ_t^\top\g_*(\tt-\t_*)}{t}=\g_*^\top \Sigma_Z\g_*(\tt-\t_*) =\nabla_\theta F(\theta_t),
\end{align*}
where $\expec{\cdot}{t}$ is the conditional expectation w.r.t the filtration defined on $\{\g_1,\t_1,\g_2,\t_2,\cdots,\g_t,\tt\}$.
       
       In  reality, $\g_*$ is unknown beforehand. Hence, we estimate $\g_*$ using some online procedure and replace $\nabla_{\tt} g(\tt;X_t')$ by $\gt^\top Z_t$ instead of $\g_*^\top Z_t$.
It leads to the following updates:
\begin{align}\label{eq:thetaupdatecso}
             \t_{t+1}=\t_t-\atone\g_t^\top Z_t(X_t^\top\tt-Y_t),\quad\quad \g_{t+1}=\g_t-\betone Z_t(Z_t^\top\gt-X_t^\top). 
    \end{align}
    A closer inspection reveals that the updates in \eqref{eq:thetaupdatecso} can diverge until $\gt$ is close enough to $\g_*$. It is easy to see this fact from the following expansion of $\ttone-\t_*$. We have
\begin{align*}
    \ttone-\t_*=&\hQt(\tt-\t_*)+\atone(\gt-\g_*)^\top\Sigma_{ZY}+\atone D_t\t_*+\atone\gt^\top\xi_{Z_t}\g_*(\tt-\t_*)\\    &+\atone\gt^\top\xi_{Z_t}\g_*\t_*+\atone\gt^\top\xi_{Z_tY_t}-\atone \gt^\top Z_t\epstt^\top\tt,\numberthis\label{eq:thetaupdateexpandedcso}
\end{align*}
where 
\begin{align*}
\xi_{Z_t}=\Sigma_Z-Z_tZ_t^\top,\quad \xi_{Z_tY_t}=\Sigma_{ZY}-Z_tY_t, \quad
\hQt\coloneqq\lrbrac{I-\atone\gt^\top\Sigma_Z\g_*}.
\end{align*}

However, the matrix $\gt^\top\Sigma_Z\g_*$ may not be positive semi-definite, even if $\Sigma_Z$ is positive definite. Thus the negative eigenvalues associated with $\gt^\top\Sigma_Z\g_*$ might cause the $\theta_t$ iterates to first diverge, before eventually converging as $\gt$ gets closer to $\g_*$. 
We illustrate this intuition in a simple experiment in Figure~\ref{fig:wrapfig}. To resolve this issue, we propose Algorithm~\ref{alg:one_sample_onlineIV}, where we replace $g(\tt,X_t) =X_t^\top \tt$ with $ Z_t^T\gt\tt$ in \eqref{eq:thetaupdatecso}.
With such a modification, in the corresponding decomposition for $\theta_{t+1} - \theta_*$ (see \eqref{eq:thetaupdateexpanded}), we have $\hQt=\lrbrac{I-\atone\gt^\top\Sigma_Z\gt}$, where the matrix product $\gt^\top\Sigma_Z\gt$ is always positive semi-definite. 
Hence, with a properly chosen stepsize $\alpha_{t}$ we could quantify the convergence of $\theta_t$ to $\theta_*$ non-asymptotically. Nevertheless, assuming a warm-start condition on $\theta_0$, we also show the convergence of  \eqref{eq:thetaupdatecso},  in Appendix~\ref{sec:csoconvergence} for completeness.\\


\begin{figure}[t]
  \centering
    \includegraphics[width=0.50\textwidth]{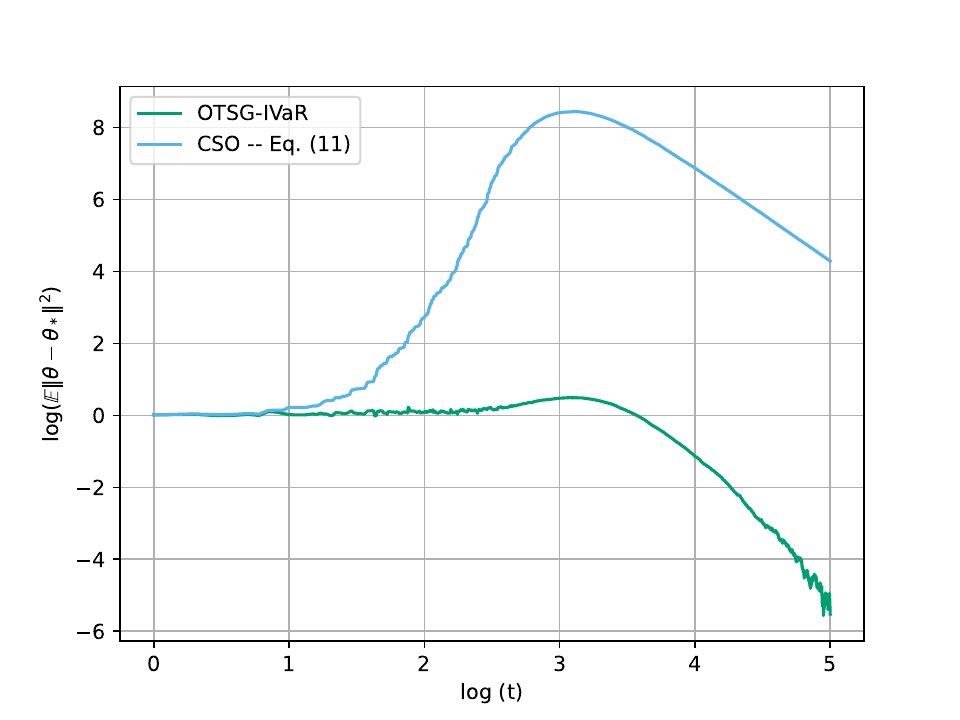}
\caption{\eqref{eq:thetaupdatecso} can initially diverge before converging eventually, leading to a worse performance in practical settings compared to Algorithm~\ref{alg:one_sample_onlineIV}. See Appendix~\ref{sec:fig3app} for the experimental setup.}.
\label{fig:wrapfig}
\end{figure}

\textbf{Algorithm and Analysis.} Based on the intuition, we present Algorithm~\ref{alg:one_sample_onlineIV}. One could interpret the algorithm as the SGD analogy of the offline 2SLS algorithm~\citep{angrist1995two}. It is also related to the framework of non-linear two-stage stochastic approximation algorithms~\citep{doan2020finite,dalal2018finite,mokkadem2006convergence}; albeit the updates of $\theta_t$ and $\gamma_t$ are coupled since both updates use $Z_t$. Furthermore, the dependency between the randomness between the two stages in the IVaR problem, makes the analysis significantly different and more challenging from the classical analysis of two-stage algorithms (see below Theorem~\ref{th:mainthetaconv} for additional details). Finally, while Algorithm~\ref{alg:one_sample_onlineIV} is designed for linear models, the intuition behind the method is also applicable to non-linear models (i.e., between $Z$ and $X$, and $X$ and $Y$). We focus on linear models in this work in order to derive our theoretical results. A detailed treatment of the nonlinear case (for which the analysis is significantly nontrivial) is left for future work. We make the following additional assumptions for the convergence analysis of  Algorithm~\ref{alg:one_sample_onlineIV}.

\begin{algorithm}[t!]
	\caption{One-Sample Two-stage Stochastic Gradient-IVarR (\texttt{OTSG-IVaR})}
	\label{alg:one_sample_onlineIV}
	\begin{algorithmic}[1]
		\REQUIRE  Stepsizes~$\{\alpha_t\}_{t}$, $\{\beta_t\}_{t}$, initial iterates~$\gamma_1,\theta_1$.
            \FOR{$t=1,~2,~\cdots$ \do} 
            \STATE Sample $Z_t$, sample $X_t$ from $\PP_{X\mid Z_t}$, Sample $Y_t$ from $\PP_{Y \mid X_t}$. 
            \STATE Update 
            \begin{align}
             &\t_{t+1}=\t_t-\atone\g_t^\top Z_t(Z_t^\top\g_t\t_t-Y_t),\label{eq:thetaupdate}\\
    &\g_{t+1}=\g_t-\betone Z_t(Z_t^\top\g_t-X_t^\top)\label{eq:gammaupdate}.   
    \end{align}
            \ENDFOR
	\end{algorithmic}
\end{algorithm}




\begin{assumption}\label{as:boundedcovarnoisevar}
For some constants $C_z,C_{zy}>0$, we have the following bounds on the fourth moments: 
    \begin{align*}
        \expec{\norm{\Sigma_Z-ZZ^\top}^4}{}\leq C_z d_z^{\sigmar_5}, \quad \expec{\norm{\Sigma_{ZY}-ZY}^4}{}\leq C_{zy}d_z^{\sigmar_6}, \quad \vartheta:=\max\{\vartheta_5,\vartheta_6\}.\numberthis\label{eq:boundedcovarnoisevar}
    \end{align*}
\end{assumption}
\begin{assumption}\label{as:sigmazposdef}
 There exist constants $0<\mu_Z \leq \lambda_Z <\infty$ such that $\mu_Z I_{d_z}\preceq \Sigma_Z\preceq  \lambda_{Z}I_{d_z} $. 
\end{assumption}

The above conditions are rather mild moment conditions, similar to Assumption~\ref{aspt: var_general}, and could be easily verified for the linear model setting we consider.

\begin{assumption}\label{as:gtbounded}
 $\{\gt\}_t$ is within a compact set of diameter $C_\g d_z^\varkappa$ for some constants $C_\g>0$, $\varkappa\geq 0$.
\end{assumption}
We emphasize that Assumption~\ref{as:gtbounded} is only  for the uncoupled sequence $\gamma_t$, which is an SGD sequence for solving a strongly-convex problem. It holds easily in various cases, for example by projecting the iterates onto any compact sets or a sufficiently large ball containing $\gamma^*$. It is also well-known that, without any projection operations, $\{\gt\}_t$ sequence is almost surely bounded \cite{polyak1992acceleration} under our assumptions. Finally, similar assumptions routinely appear in the analysis of SGD algorithms in various related settings; see, for example, \cite{tseng1998incremental,gurbuzbalaban2019convergence,haochen2019random,nagaraj2019sgd,ahn2020sgd,rajput2020closing}. 

We now present our result on the convergence of $\{\tt\}_t$ below in Theorem~\ref{th:mainthetaconv} (see Appendix~\ref{sec:pfmainthetaconv} for the proof). 
In comparison to Theorem~\ref{thm:convergence_linear_two_sample} (regarding Algorithm~\ref{alg:two_sample_SGD}), we highlight that Theorem~\ref{th:mainthetaconv} provides an any-time guarantee, as the total number of iterations is not required in advance by Algorithm~\ref{alg:one_sample_onlineIV}. 

\begin{theorem}\label{th:mainthetaconv} 
     Suppose Assumptions~\ref{aspt: scvx}, \ref{aspt: independence} (without $X'_t)$, \ref{as:boundedcovarnoisevar}, \ref{as:gtbounded}, and \ref{as:sigmazposdef} hold. In Algorithm~\ref{alg:one_sample_onlineIV}, for any $\iota>0$, set $\at=C_\alpha t^{-1+\iota/2}$ and $\bet= C_\beta t^{-1+\iota/2},$ where $C_\alpha= \min\{0.5d_z^{-4\varkappa-\vartheta/2}\lambda_Z^{-1}C_\g^{-2},0.5(\norm{\g_*}\lambda_Z)^{-2}\}$, and $C_\beta=\mu^2 d_z^{-1-2\varkappa}/128$. Then, we have
     \begin{align*}
         \expec{\norm{\tt-\t^*}^2}{}=O\lrbrac{\frac{1}{t^{1-\iota}}}.
     \end{align*}
\end{theorem}
\begin{remark}
In Theorem~\ref{th:mainthetaconv}, we present the step-size choices for the fastest rate of convergence. In the proof of Theorem~\ref{th:mainthetaconv} (see Appendix~\ref{sec:pfmainthetaconv}), we show that convergence can be guaranteed for a range of step-sizes given by $\at=C_\alpha t^{-a}$, $\bet= C_\beta t^{-b},$ where $1/2<a,b<1$, $b>2-2a$ with corresponding rate being $
         \expec{\norm{\tt-\t^*}^2}{}=O(\max\{t^{-b(2-(1-\iota/2)^{-1})},t^{-a}\log(2/\iota-1)\})
     $. In particular, one requires $a,b<1$ to ensure $(\at-\atone)/\at=o(\at)$, and $(\bet-\betone)/\bet=o(\bet)$, as is standard in  stochastic approximation literature (see, for example,  \cite{chen2020statistical,polyak1992acceleration}). 
\end{remark}
\textbf{Proof Techniques.} The major challenge towards the  convergence analysis of $\{\tt\}_t$ lies in the interaction term $\gt Z_tZ_t^\top\gt\tt$ between $\gamma_t$ and $\theta_t$ in \eqref{eq:thetaupdate}. This multiplicative interaction term leads to  an involved dependence between the noise in the stochastic gradient updates for the two stages. Such a dependence has not been considered in existing analysis of non-linear two time-scale algorithms \citep{mokkadem2006convergence,maei2009convergent,dalal2018finite,doan2020finite,xu2021sample,wang2021non,doan2022nonlinear}. In addition, \cite{doan2022nonlinear} considers the case when the noise sequence is not only independent of each other but also independent of iterate locations. Furthermore, they assumes (see their Assumption 3) that the condition in Assumption~\ref{aspt: scvx} holds for all $\gamma$ whereas Assumption~\ref{aspt: scvx} only needs to hold for $\g_*$, that is much milder. Similarly, many works (for example, Assumption 1 in \cite{wang2021non}, Assumption 2 in \cite{xu2021sample} and Theorem 2 in \cite{maei2009convergent}) assume that the iterates of both stages are bounded in a compact set and consequently, and hence the variance of the stochastic gradients are also uniformly bounded.




In our setting, firstly, the stochastic gradient in~\eqref{eq:thetaupdate}, evaluated at $(\tt,\gt)$ is biased:
    \begin{align*}
       &\expec{ \gt^\top Z_t(Z_t^\top\gt\tt-Y_t)}{t,Z_t}=\expec{\gt^\top Z_t(Z_t^\top\g_t\tt-Z_t^\top\g_*\t_*)}{t,Z_t}=\expec{\gt^\top \Sigma_Z(\g_t\tt-\g_*\t_*)}{t}\\
       =&\gt^\top \Sigma_Z\g_t(\tt-\t_*)+\gt^\top \Sigma_Z(\gt-\g_*)\t_* \neq \g_*^\top \Sigma_Z\g_*(\tt-\t_*) = \nabla_\theta F(\theta_t).
    \end{align*}
Furthermore, even under Assumption~\ref{as:gtbounded}, the variance of the stochastic gradient is not~\eqref{eq:thetaupdate} uniformly bounded. Overcoming these issues, in addition to the aforementioned dependence between the noise in the stochastic gradient updates for the two stages, forms the major novelty in our analysis. We proceed by noting that if $\g_*$, $\Sigma_Z$, and $\Sigma_{ZY}$ were known beforehand, one conduct deterministic gradient updates, i.e., $\tilde{\t}_{t+1}=\tilde{\t}_t-\atone\g_*^\top \left(\Sigma_Z\g_*\tilde{\t}_t-\Sigma_{ZY}\right)$,  to obtain $\t_*$.
By standard results on gradient descent for strongly convex functions (see, for example, \cite{nesterov2013introductory}), $\{\tilde{\theta}_t\}_t$ converges exponentially fast as stated in Lemma~\ref{lm:tildethetaconvrate}. Hence, it remains to show that the trajectory of $\tt$ converges to the trajectory of $\tilde{\t}_t$. That is, defining the sequence $\delt\coloneqq \tt-\tilde{\t}_t$, our goal is to establish the convergence rate of $\expec{\lVert\delt\rVert_2^2}{}$. We first provide an intermediate bound (see Lemma~\ref{lm:deltintermedbound})  and then progressively sharpen to a tighter bound (see Lemma~\ref{lm:improvedfinaldeltrate}). In doing so, it is also required to show that $\expec{\norm{\tt}^4}{}$ is bounded, which we prove in Lemma~\ref{lm:thetabounded}. The proof of Lemma~\ref{lm:thetabounded} is non-trivial and requires carefully chosen stepsizes satisfying $\sum_{t=1}^\infty(\atsqr+\at\sqrt{\bet})<\infty$.\vspace{-0.05in}

%% file: Text/5_numerical_studies.tex
\section{Numerical Experiments}\label{sec:exp}
\vspace{-0.05in}
\textbf{Experiments for Algorithm \ref{alg:two_sample_SGD} (\texttt{TOSG-IVaR}).} We first consider the following problem, in which $(Z, X, Y)$ is generated via
\begin{align*}
    Z \sim \mathcal{N}(0, I_{d_z}),\ X = \phi(\gamma_*^\top Z) + c\cdot (h + \epsilon_x),\ Y = \theta_*^\top X + c\cdot(h_1 + \epsilon_y),
\end{align*}
where $c>0$ is a scalar to control the variance of the noise vector, and $h_1$ is the first coordinate of $h$. The noise vectors (or scalar) $h, \epsilon_x, \epsilon_y$ are independent of $Z$, and we have $h\sim \mathcal{N}(\mathbf{1}_{d_x}, I_{d_x})$, $\epsilon_x \sim \mathcal{N}(0, I_{d_x})$,$\epsilon_y \sim \cN(0, 1)$. In each iteration, one tuple $(X, X', Y)$ is generated and used to update $\theta_t$ according to Algorithm \ref{alg:two_sample_SGD}. We set $(d_x, d_z)\in \{(4, 8), (8, 16)\}$, $c\in \{0.1, 1.0\}$, and $\phi(s)\in \{s, s^2\}$. We repeat each setting 50 times and report the curves of $\EE[\norm{\theta_t - \theta_*}^2]$ in Figure \ref{fig: algo1}, where the expectation is computed as the average of $\norm{\theta_t - \theta_*}^2$ of all trials, and the shaded region represents the standard deviation. The first row and the second row correspond to $\phi(s) = s$ and $\phi(s) = s^2$ respectively. Here, $c=0.1$ for odd columns and $c=1.0$ for even columns. We have $(d_x, d_z) = (4, 8)$ for the first two columns and $(d_x, d_z) = (8, 16)$ for the last two columns. Empirically, we can observe that our Algorithm \ref{alg:two_sample_SGD} performs well across all different settings.

\begin{figure}
    \centering
    \subfigure[]{\includegraphics[width=0.233\textwidth]{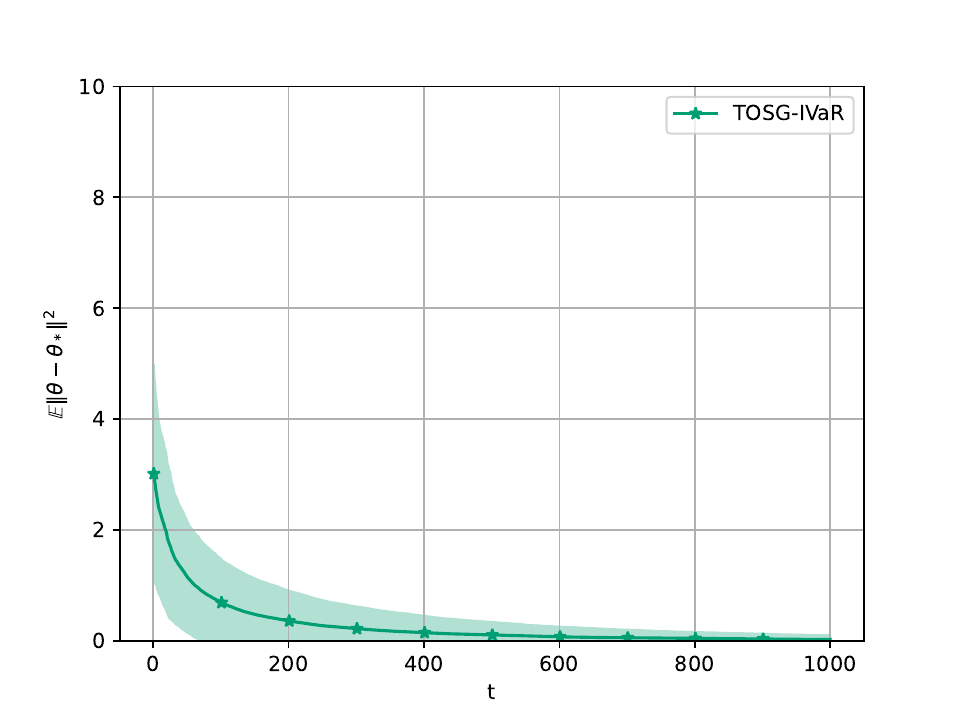}}
    \subfigure[]{\includegraphics[width=0.233\textwidth]{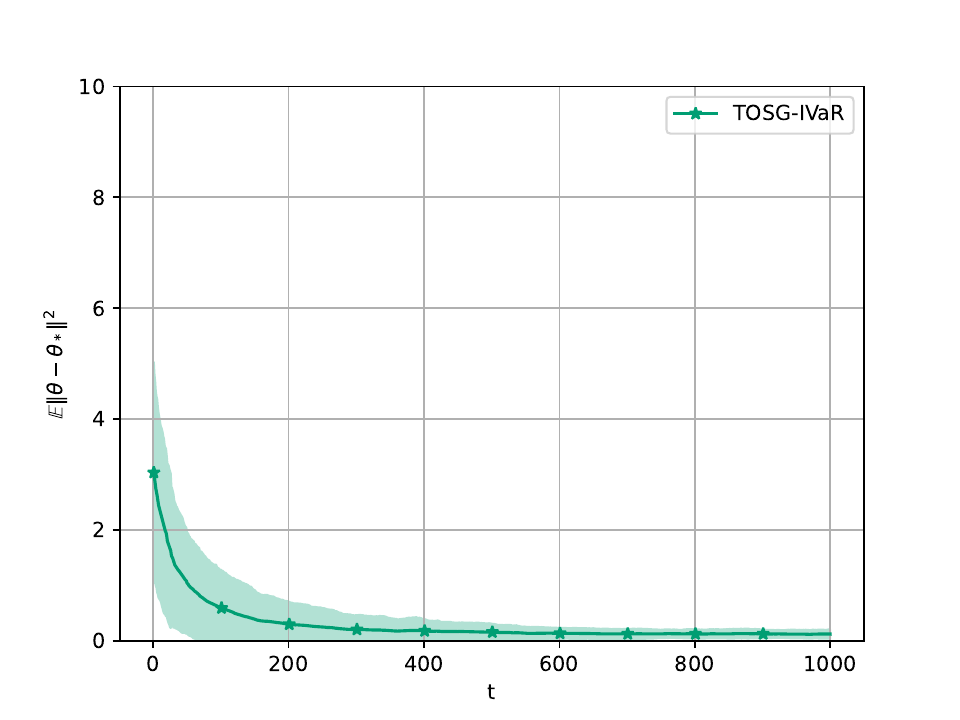}}
    \subfigure[]{\includegraphics[width=0.233\textwidth]{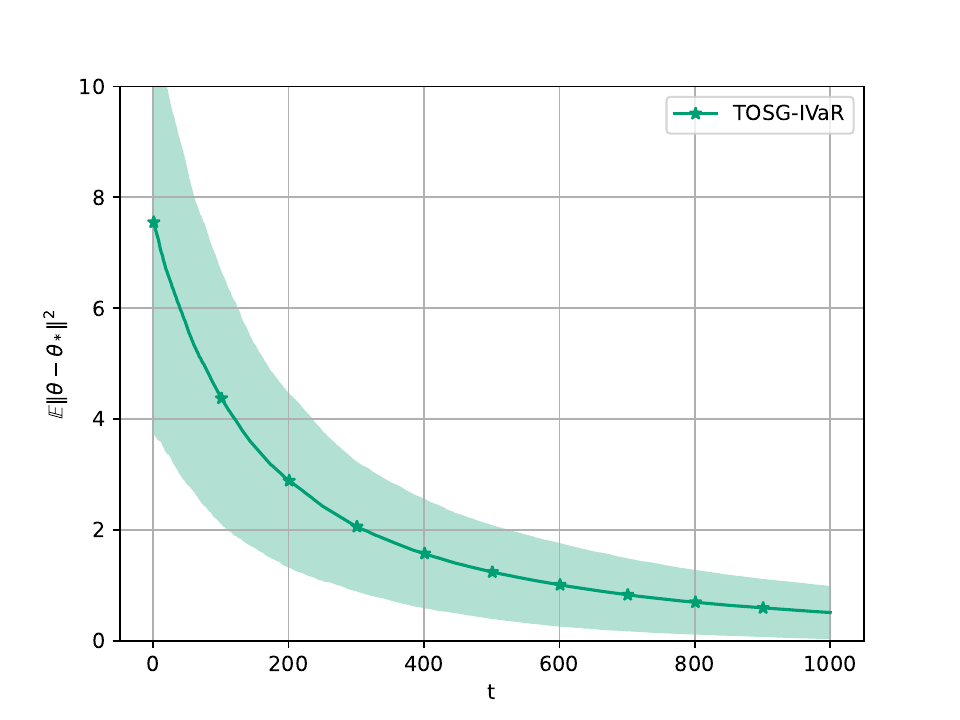}}
    \subfigure[]{\includegraphics[width=0.233\textwidth]{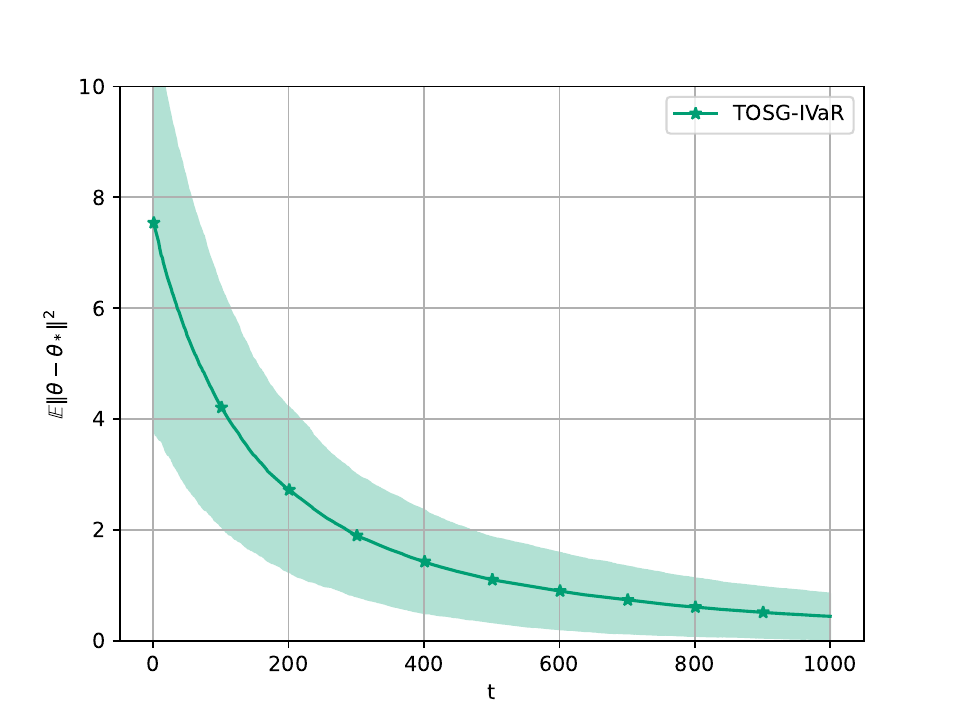}}
    \subfigure[]{\includegraphics[width=0.233\textwidth]{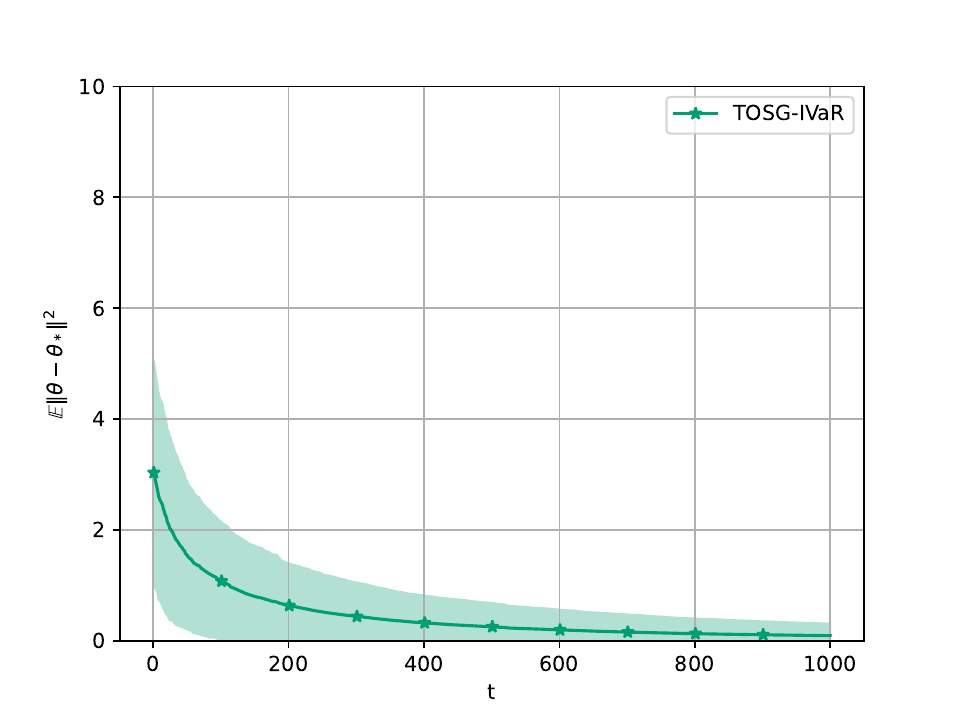}}
    \subfigure[]{\includegraphics[width=0.233\textwidth]{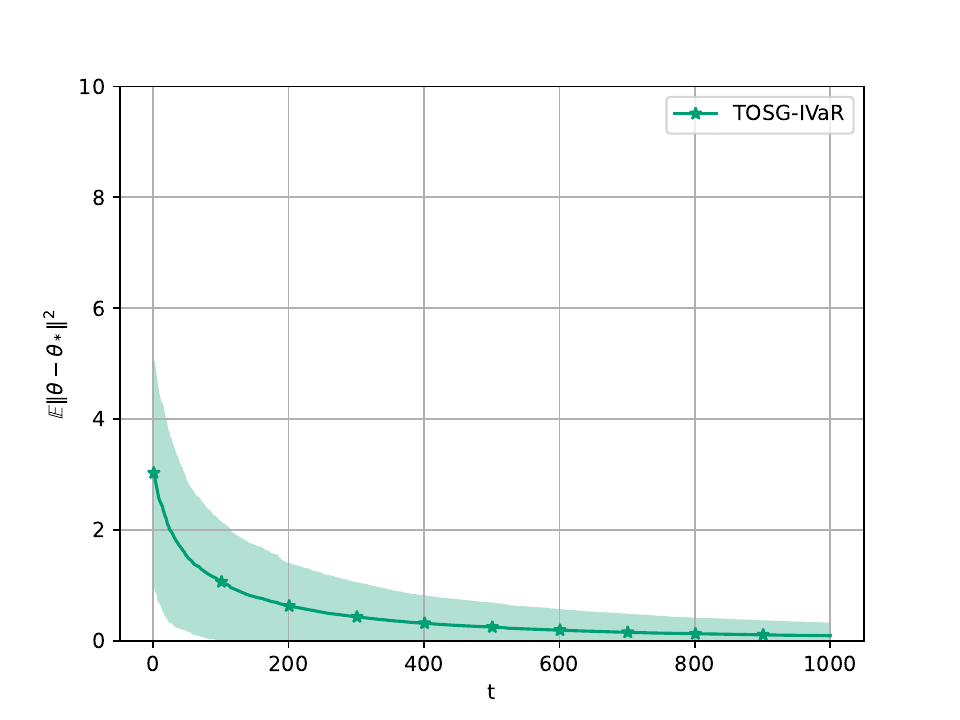}}
    \subfigure[]{\includegraphics[width=0.233\textwidth]{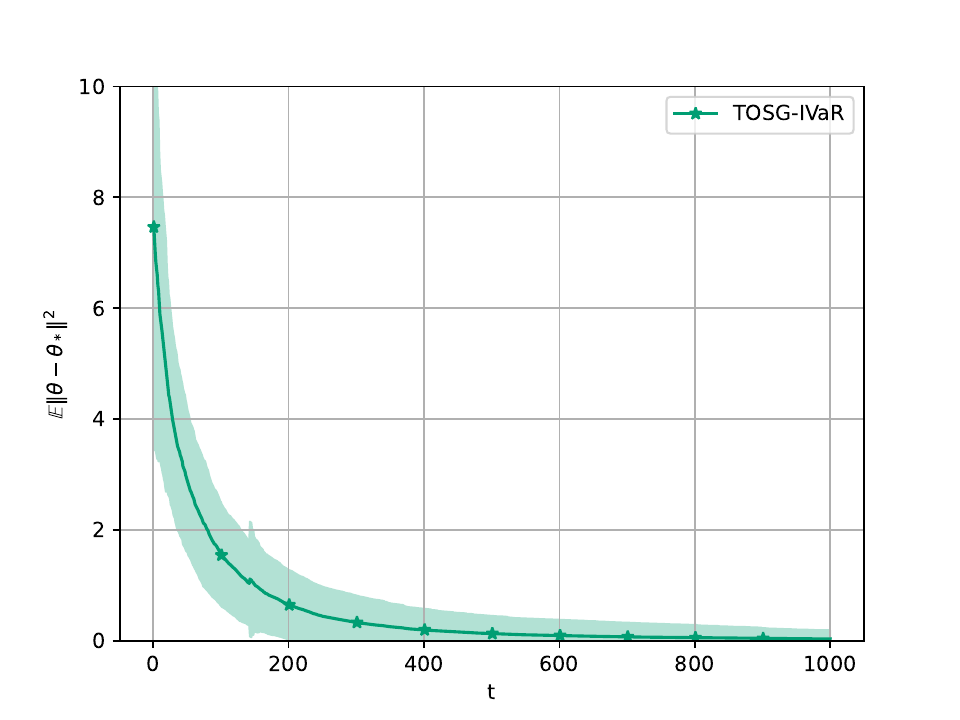}}
    \subfigure[]{\includegraphics[width=0.233\textwidth]{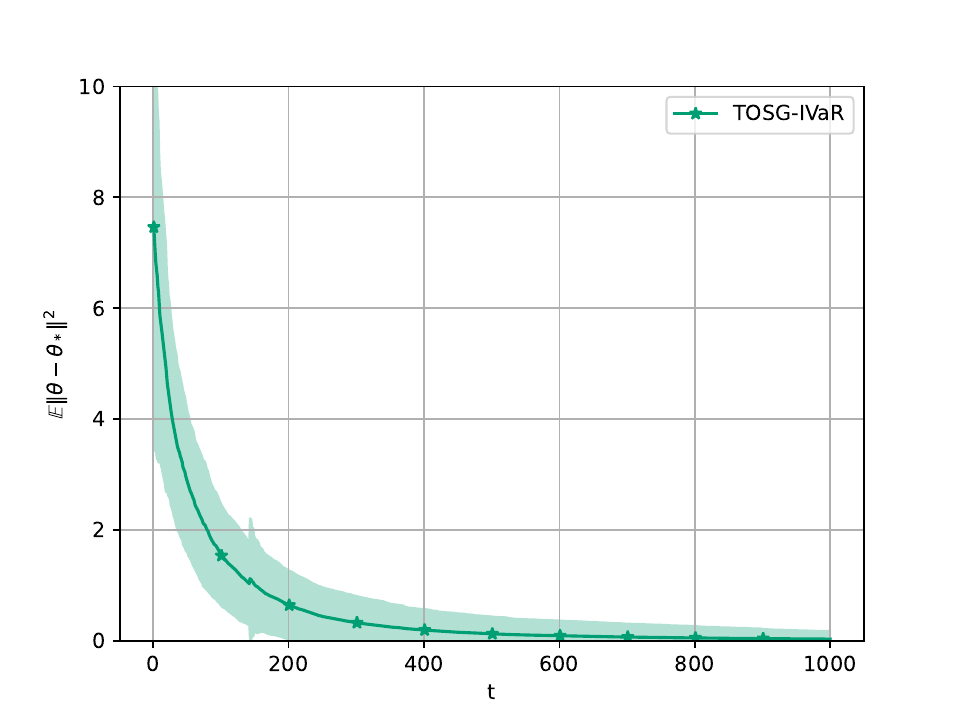}}
    \caption{$\EE[\norm{\theta_t - \theta_*}^2]$ of Algorithm \ref{alg:two_sample_SGD} under different settings detailed in Section~\ref{sec:exp}.}
    \label{fig: algo1}
\end{figure}

\noindent\textbf{Experiments for Algorithm \ref{alg:one_sample_onlineIV} (\texttt{OTSG-IVaR}).} 
Next, we compare our Algorithm \ref{alg:one_sample_onlineIV} as well as its variant and Algorithm 1 in \cite{della2023online}. We write ``OTSG-IVaR'', ``CSO -- Eq. (11)'' and ``[DVB23]'' to represent Algorithm \ref{alg:one_sample_onlineIV}, Algorithm \ref{alg:one_sample_onlineIV} with the updates replaced by \eqref{eq:thetaupdatecso} and Algorithm 1 in \cite{della2023online} (see Appendix~\ref{alg:pseudoools}).
We follow simulation settings similar to \cite{della2023online}:
\begin{align*}
    &Y=\theta_*^\top X +\nu, \qquad X=\gamma_*^\top Z+\epsilon, \qquad
    \epsilon=\sigma_\epsilon \cN(0,I_{d_x}),\qquad \nu=\rho\epsilon_1+\cN(0,0.25), \numberthis\label{eq:exptsetting}
\end{align*}
where $\epsilon_1$ is the first coordinate of $\epsilon$, $\theta_*\in\mathbb{R}^{d_x}$ is a unit vector chosen uniformly randomly, and $\gamma_*\in\mathbb{R}^{d_z\times d_x}$ where $\gamma_{ij}=0$ for $i\neq j$, and $\gamma_{ij}=1$ for $i= j$, $i=1,2,\cdots,d_x$, and $j=1,2,\cdots,d_z$. Here $\rho$ controls the level of endogeneity in the model. We compare the performance of Algorithm~\ref{alg:one_sample_onlineIV} with \eqref{eq:thetaupdatecso}, and O2SLS \citep{della2023online} for $\rho=1,4$, and $\sigma_\epsilon=0.5,1$. By varying $\sigma_\epsilon$ we control the correlation between $X$ and $Z$. We consider two settings $(d_x,d_z)=(1,1)$, and $(d_x,d_z)=(8,16)$. As performance metric, in Figure \ref{fig: algo2_dist} we plot $\expec{\norm{\tt-\t_*}^2}{}$ where the $\expec{\cdot}{}$ is approximated by averaging over 50 trials, and both axes are in $\log$ scale (base 10). We also show, in Figure \ref{fig: algo2_mse}, the convergence of the test Mean Squared Error (MSE) evaluated over 400 test samples to the best possible Test MSE where $\t_*$ and $\g_*$ are known beforehand. For Figures \ref{fig: algo2_dist} and \ref{fig: algo2_mse}, the first row and second row corresponds to $(d_x, d_z) = (1, 1)$ and $(d_x, d_z) = (8, 16)$ respectively, and $\sigma_{\epsilon} = 0.5$ in odd columns and $\sigma_{\epsilon} = 1.0$ in even columns. We have $\rho = 1.0$ for the first two columns and $\rho = 4.0$ for the last two columns.
We can observe that O2SLS has much larger variance in different settings, while our algorithms perform consistently well in all settings.
\vspace{-0.05in}

\begin{figure}
    \centering
    \subfigure[]{\includegraphics[width=0.233\textwidth]{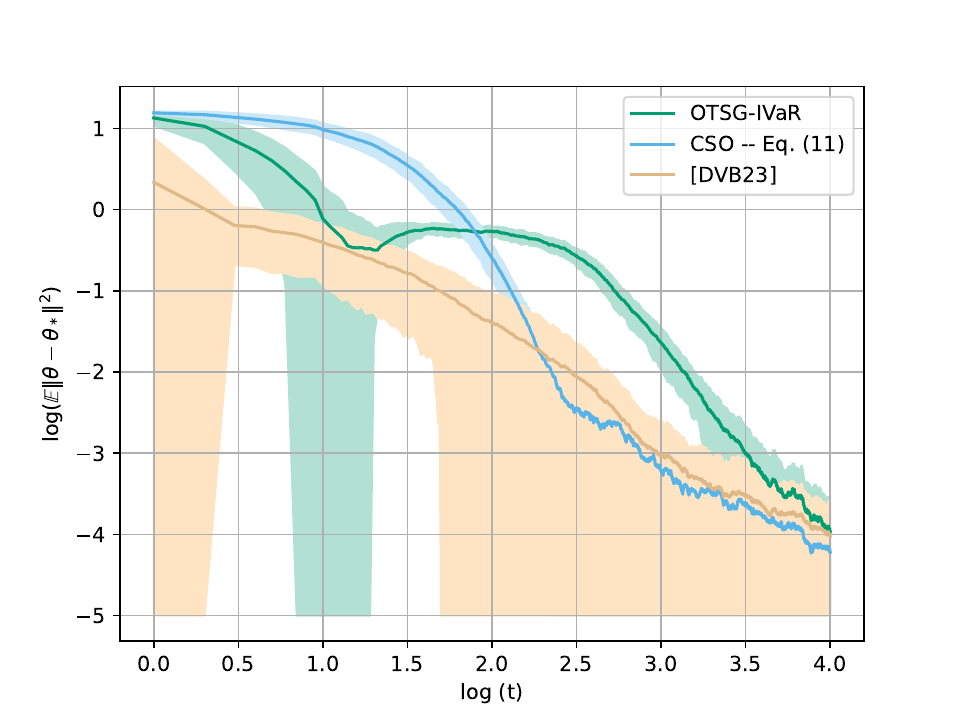}}
    \subfigure[]{\includegraphics[width=0.233\textwidth]{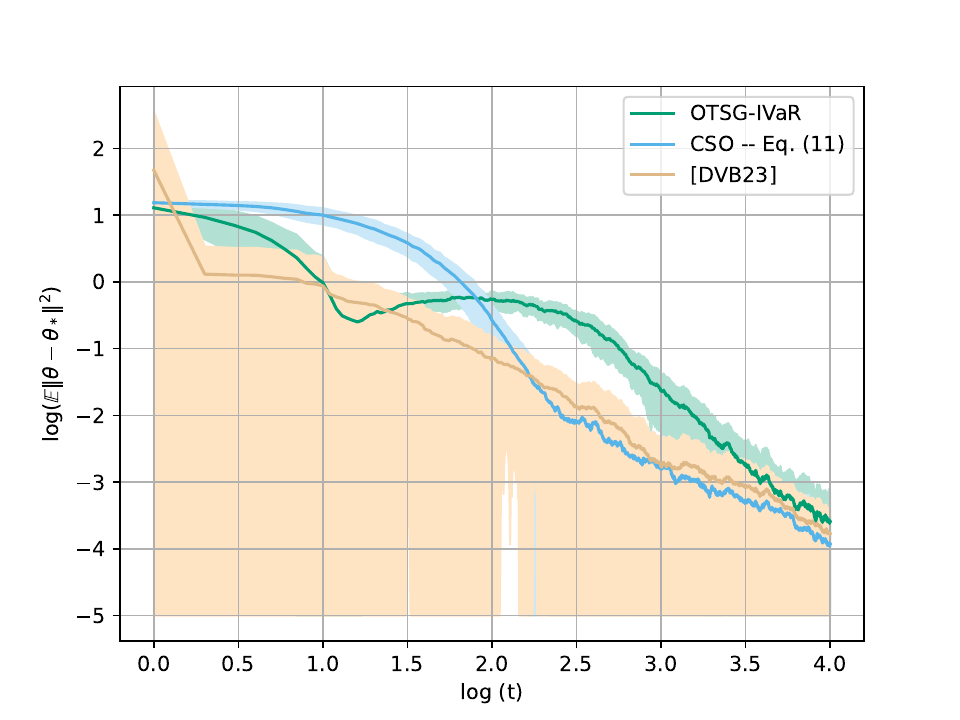}}
    \subfigure[]{\includegraphics[width=0.233\textwidth]{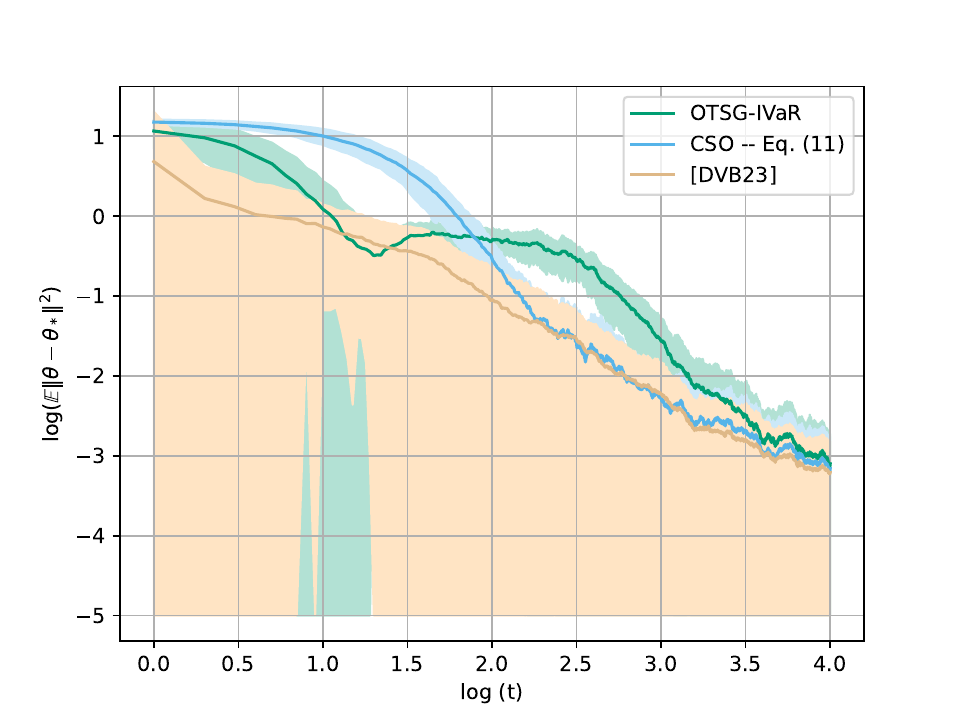}}
    \subfigure[]{\includegraphics[width=0.233\textwidth]{figures/Algo2/dist_to_optima/DistancetoOfflineOptima_dx1_dz1_rho4_sig_epsilon0.500000.pdf}}
    \subfigure[]{\includegraphics[width=0.233\textwidth]{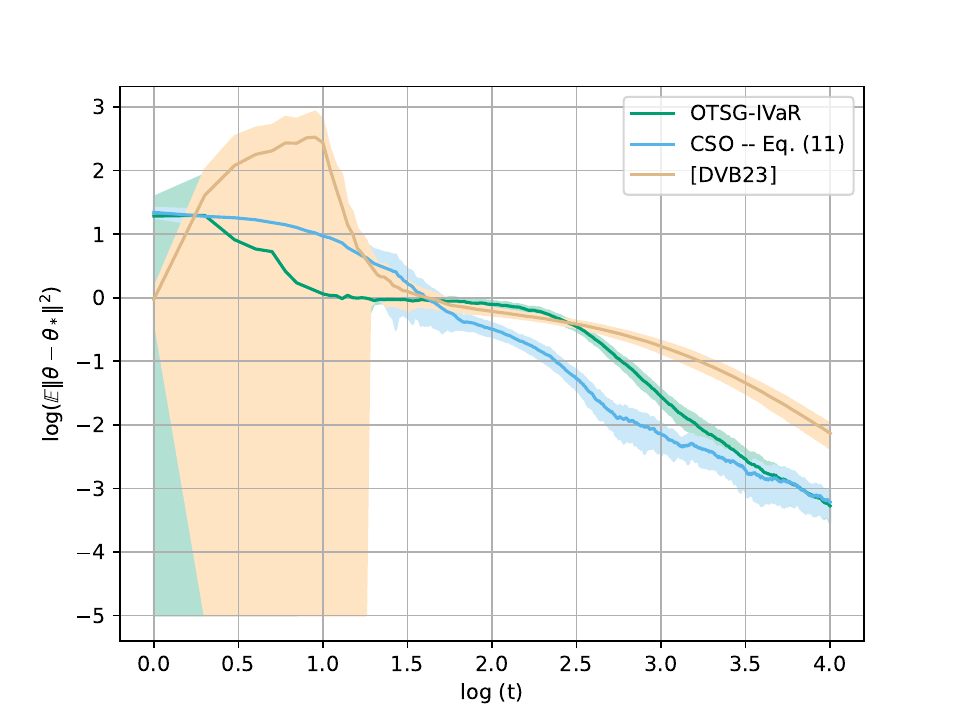}}
    \subfigure[]{\includegraphics[width=0.233\textwidth]{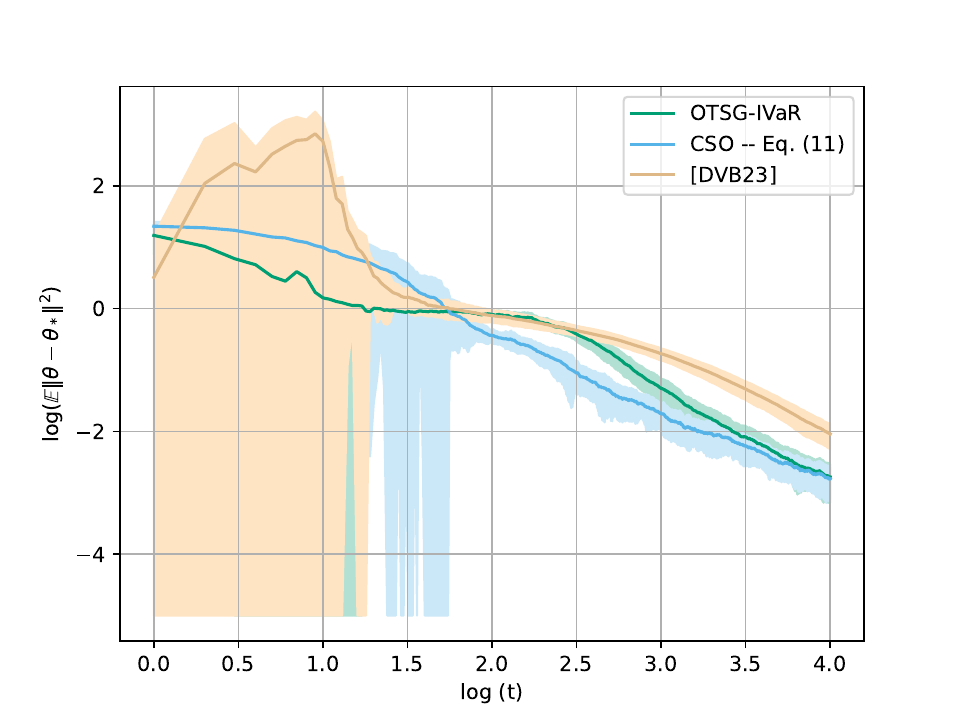}}
    \subfigure[]{\includegraphics[width=0.233\textwidth]{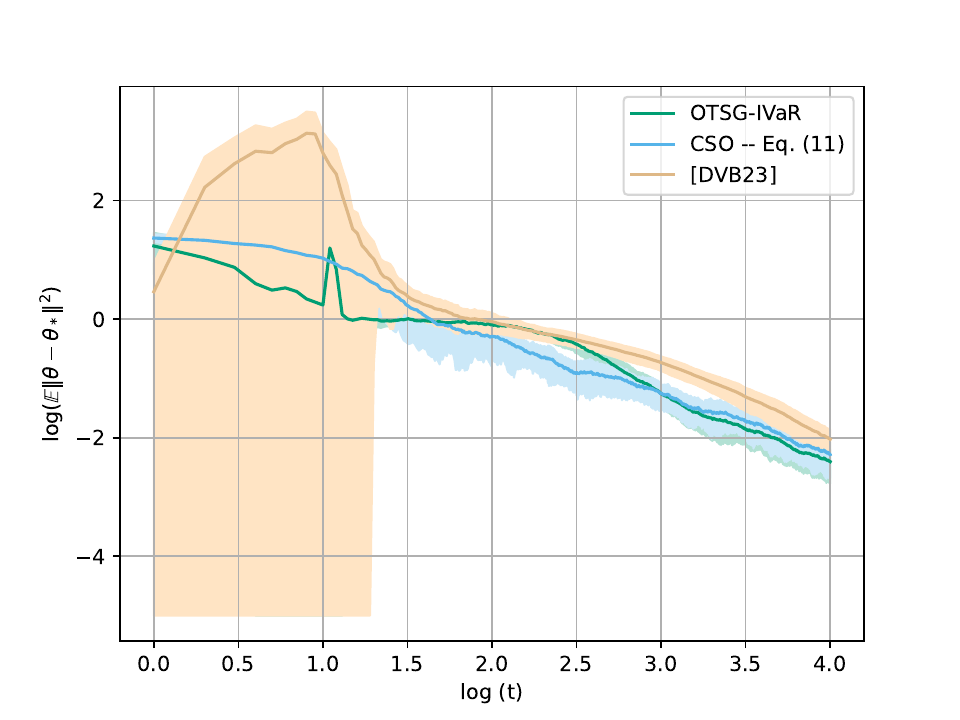}}
    \subfigure[]{\includegraphics[width=0.233\textwidth]{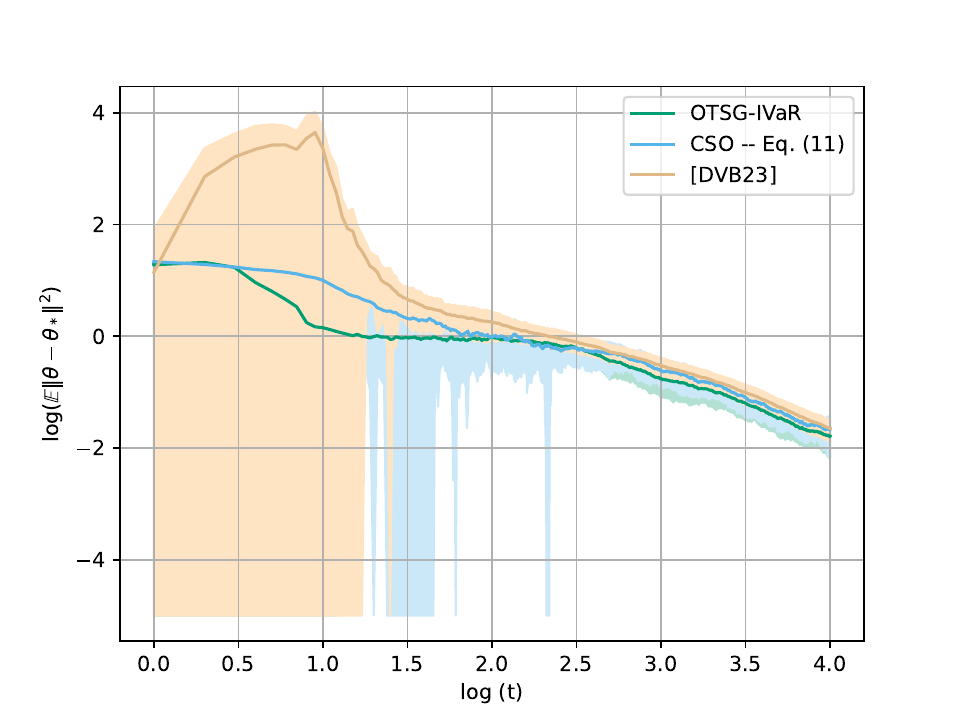}}
    \caption{Comparison of $\EE[\norm{\theta_t - \theta_*}^2]$ ($\log$-$\log$ scale) for Algorithm~\ref{alg:one_sample_onlineIV}, Eq.~\ref{eq:thetaupdatecso} and~\cite{della2023online}.}
    \label{fig: algo2_dist}
    \vspace{-0.1in}
\end{figure}

\begin{figure}
    \centering
    \subfigure[]{\includegraphics[width=0.233\textwidth]{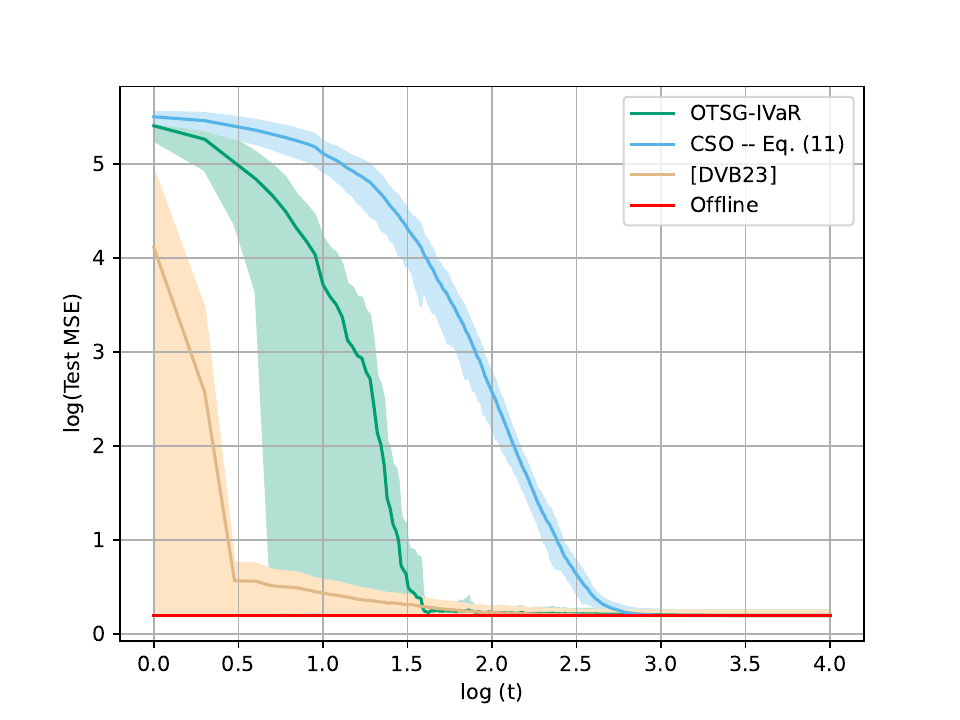}}
    \subfigure[]{\includegraphics[width=0.233\textwidth]{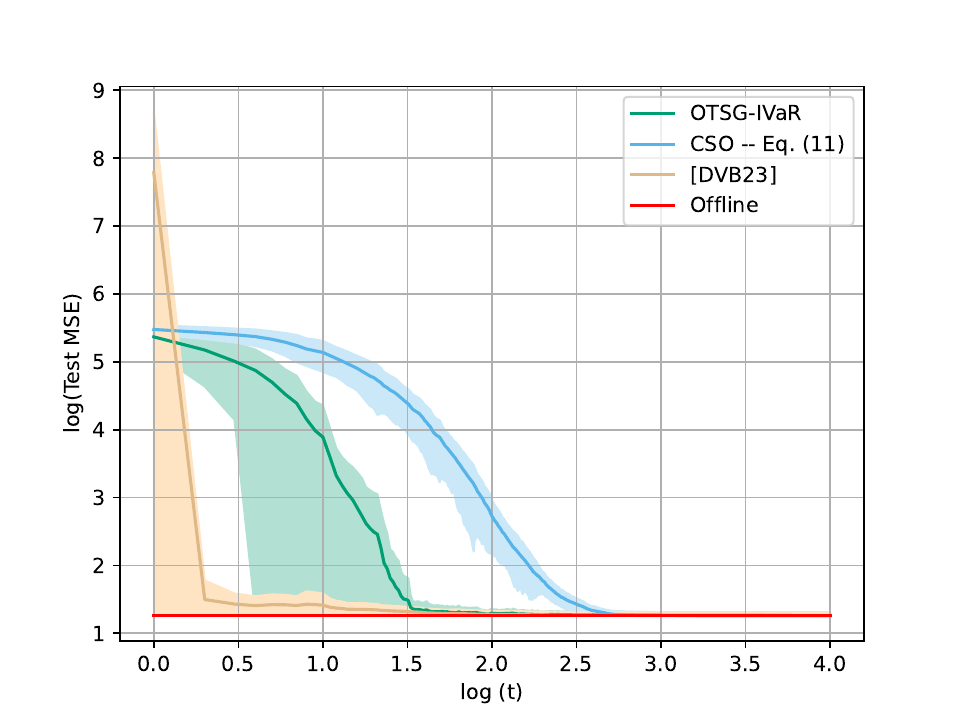}}
    \subfigure[]{\includegraphics[width=0.233\textwidth]{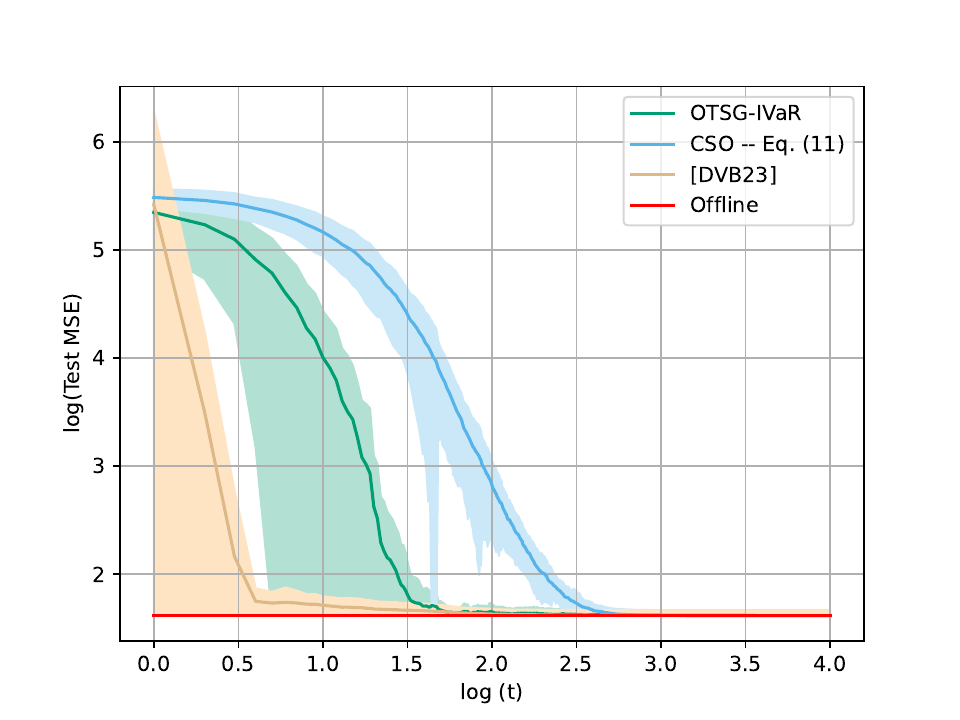}}
    \subfigure[]{\includegraphics[width=0.233\textwidth]{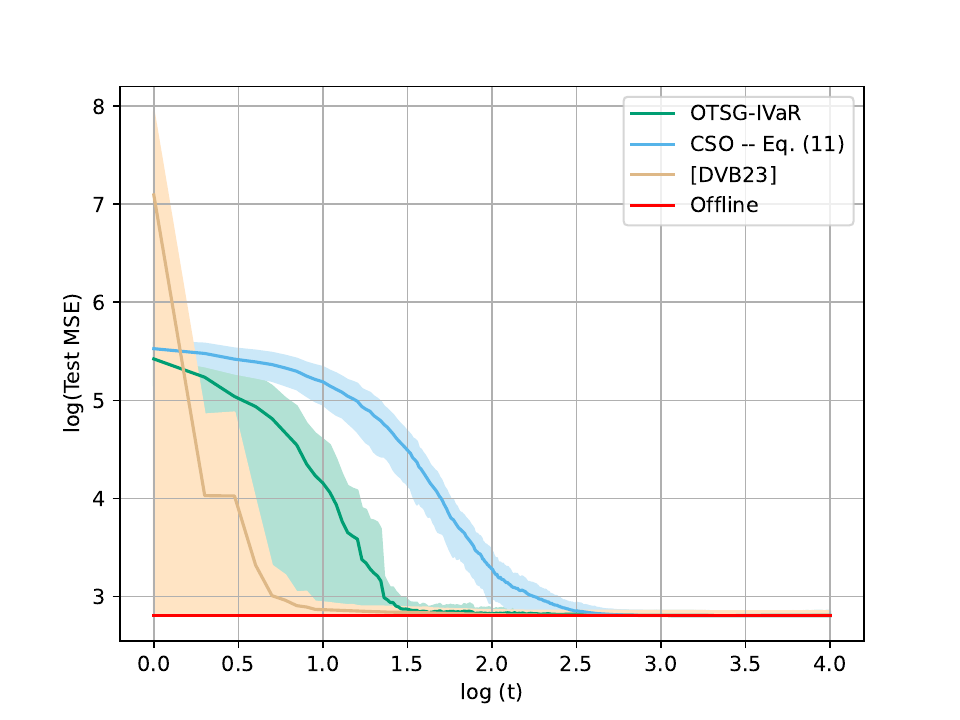}}
    \subfigure[]{\includegraphics[width=0.233\textwidth]{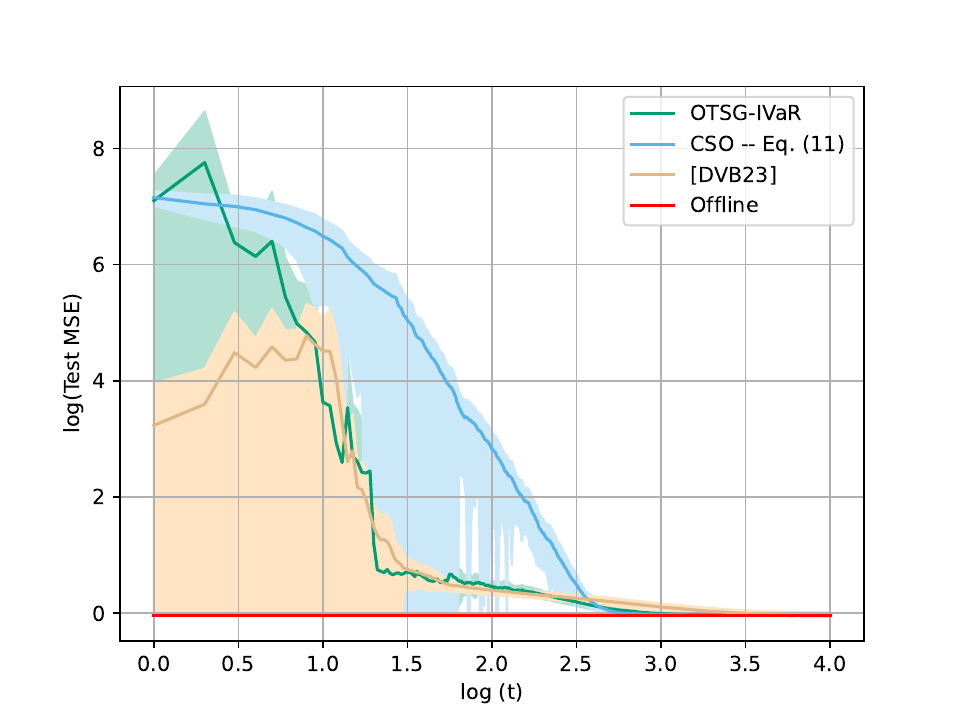}}
    \subfigure[]{\includegraphics[width=0.233\textwidth]{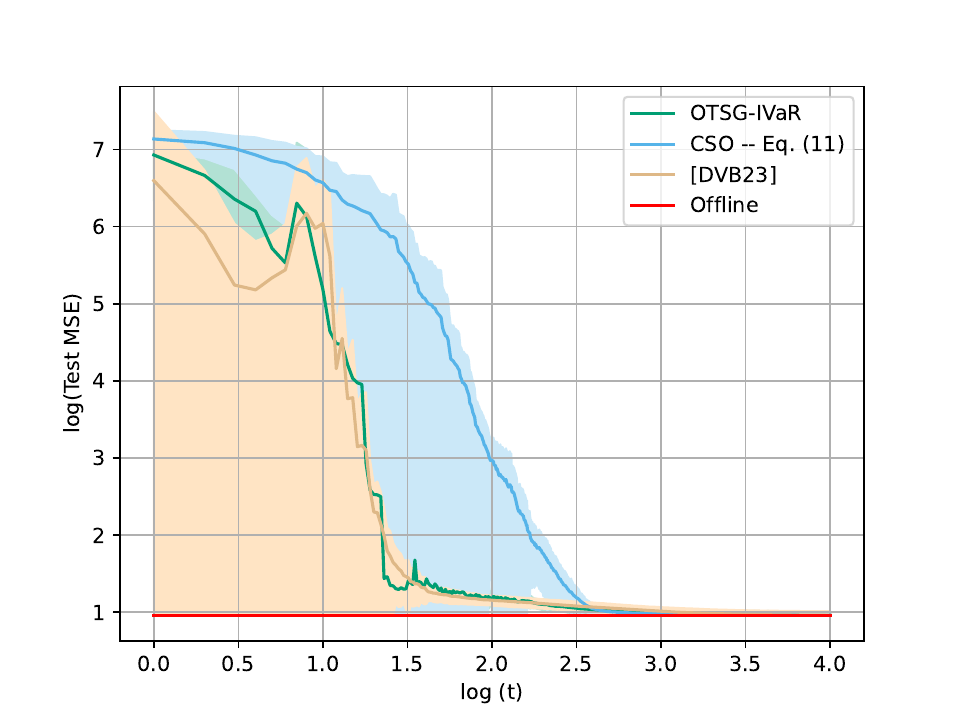}}
    \subfigure[]{\includegraphics[width=0.233\textwidth]{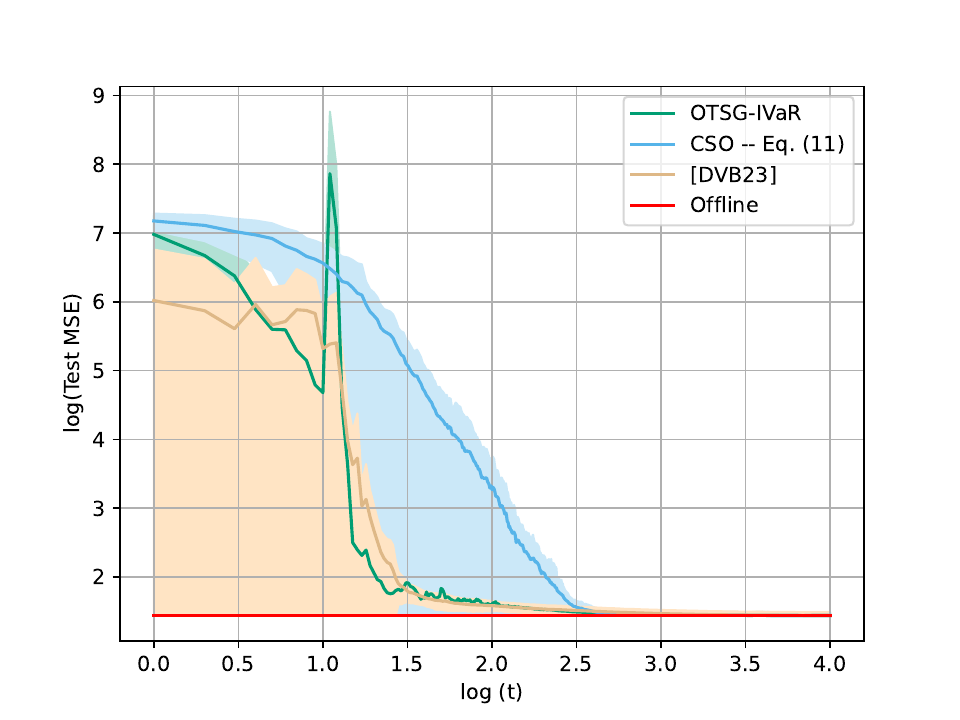}}
    \subfigure[]{\includegraphics[width=0.233\textwidth]{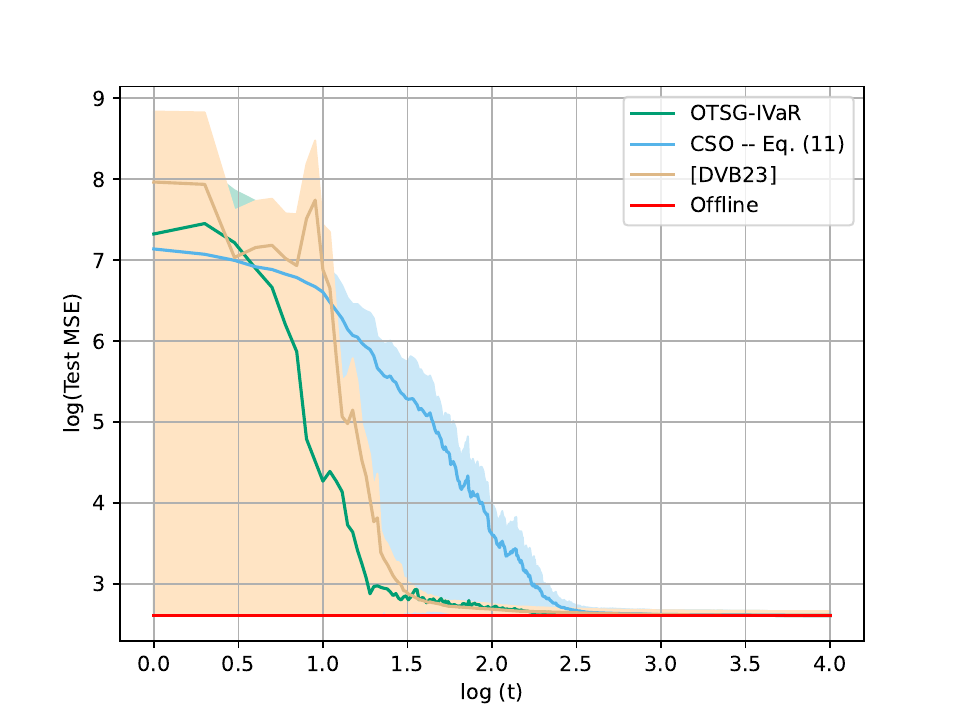}}
    \caption{Comparison of Test MSE ($\log$-$\log$ scale) for Algorithm~\ref{alg:one_sample_onlineIV}, Eq.~\ref{eq:thetaupdatecso} and~\cite{della2023online}.}
    \label{fig: algo2_mse}
    \vspace{-0.1in}
\end{figure}

%% file: Text/6_conclusion.tex
\section{Conclusion}
\vspace{-0.1in}
We presented streaming algorithms for least-squares IVaR based on directly solving the associated conditional stochastic optimization formulation in~\eqref{problem:IV_parameterized}. Our algorithms have several benefits, including avoidance of mini-batches and matrix inverses. We show that the expected rates of convergences for the proposed algorithms are of order $\mathcal{O}(\log T/T)$ and $\mathcal{O}(1/T^{1-\iota})$, for any $\iota>0$, under the availability of two-sample and one-sample oracles, respectively. As future work, it is interesting to develop streaming inferential methods for IVaR. Leveraging related works for the vanilla SGD~\citep{polyak1992acceleration, anastasiou2019normal, shao2022berry, chen2020statistical, zhu2023online} to the setting of Algorithms~\ref{alg:two_sample_SGD} and~\ref{alg:one_sample_onlineIV}, provides a concrete direction to establish Central Limit Theorems and develop limiting covariance estimation procedures.  

%% file: Text/7_appendix.tex
\appendix


\section{Online updates of \cite{della2023online}}\label{alg:pseudoools}
For the sake of clarity, we present the O2SLS algorithm proposed in \citep[v3]{della2023online}\footnote{Note that the streaming algorithm was not present in version 1, i.e., \citep[v1]{della2023onlinev1}.} in the streaming format, without any explicit matrix inversions that we used in our experiments: 
\begin{align*}
    \ttone&=(I-U_t\gt^\top Z_tZ_t^\top\gt)\tt+U_t\gt^\top Z_t Y_t\\ 
    \gtone&=(I-V_t Z_tZ_t^\top)\gt+V_tZ_tX_t^\top\\
    U_{t+1}&=U_t-\frac{U_t\gt^\top Z_tZ_t^\top\gt U_t}{1+Z_t^\top\gt U_t \gt^\top Z_t}\\
    V_{t+1}&=V_t-\frac{V_t Z_tZ_t^\top V_t}{1+Z_t^\top V_t Z_t} \qquad V_0=\lambda^{-1} I_{d_z},
\end{align*}
where $U_t, V_t$ are two additional matrix sequences which tracks the matrix inverse of $\sum_{i=1}^t\g_i^\top Z_iZ_i^\top\g_i,$  and $(\lambda I_{d_z}+\sum_{i=1}^t Z_tZ_t^\top)$ respectively for a user defined parameter $\lambda$. As mentioned in \cite{della2023online}, we choose $\lambda=0.1$. The major difference between O2SLS and Algorithm~\ref{alg:one_sample_onlineIV} is that O2SLS takes an online two-stage regression approach to minimize a suitably defined regret whereas we take a conditional stochastic optimization point of view which requires carefully chosen step-sizes. In our Algorithm~\ref{alg:one_sample_onlineIV}, we do not need to explicitly or implicitly do matrix inverse which can potentially cause stability issues. Furthermore, unlike \cite{della2023online}, we neither assume $\sum_{i=1}^tZ_iZ_i^\top$ is invertible for all $t$ nor do we assume that $Z$ is a bounded random variable for our analysis. Finally, the per-iteration computational complexity and memory requirement of Algorithm~\ref{alg:one_sample_onlineIV} is significantly better than O2SLS; see Section~\ref{sec:periter}.

\section{Per-iteration Complexities}
\label{sec:periter}
For the linear case, i.e., the underlying relationship  between $Z$ and $X$  as well as $X$ and $Y$ are linear, Table \ref{tab:per-iteration} summarizes the per-iteration memory costs and number of arithmetic operations of the original O2SLS~\citep{della2023onlinev1}, the updated O2SLS~\citep{della2023online} that we provide a matrix form update in Appendix \ref{alg:pseudoools}, TOSG-IVaR (Alg~\ref{alg:two_sample_SGD}), and OTSG-IVaR (Alg~\ref{alg:one_sample_onlineIV}) at the $t$-th iteration. 

Notice that the original version of O2SLS~\citep{della2023onlinev1} has a per-iteration and memory cost dependent on the iteration number $t$ as it needs to use all the samples accumulated till the iteration $t$ to conduct an offline 2SLS at each iteration. The updated O2SLS~\citep{della2023online} (the algorithm that we compare to) uses samples obtained at iteration $t$ to perform the update.  Although the updated O2SLS avoids explicit matrix inversion, it is obvious that its arithmetic operations  and memory cost per iteration are larger than our TOSG-IVaR and OTSG-IVaR.

We highlight that the TOSG-IVaR, which uses two samples $X$ and $X'$ from the conditional distribution $\PP(X\mid Z)$, requires only $\cO(d_x)$ memory and arithmetic operations at each iteration. 

\begin{table}[ht]
\centering
    \caption{Memory cost and the number of arithmetic operations at iteration $t$.}
    \renewcommand*{\arraystretch}{1.4}
    \begin{tabular}{|c|c|c|}
       \hline
    Algorithm & Memory cost & Arithmetic Operations\\
    \hline \hline
        O2SLS~\citep[v1]{della2023onlinev1} & $t(d_x +  d_z)+d_z d_x+d_x$  & $\cO(d_x^3 +td_x^2 +  td_xd_z)$\\
        \hline
       O2SLS~\citep[v3]{della2023online} (Sec. \ref{alg:pseudoools})& $d_x^2+d_z^2+d_zd_x+d_x$  & $\cO(d_x^2+d_z^2+d_zd_x)$\\
        \hline
       TOSG-IVaR (our Alg~\ref{alg:two_sample_SGD})  & $d_x$ & $\cO(d_x)$\\
       \hline
       OTSG-IVaR (our Alg~\ref{alg:one_sample_onlineIV})  & $d_x d_z + d_x$ & $\cO(\min(d_x^2,d_z^2)+d_zd_x)$\\
       \hline
    \end{tabular}
    \label{tab:per-iteration}
\end{table}
For a fair comparison, we assume that two $n\times n$ matrices multiplication admits an $\cO(n^3)$ complexity, i.e., using normal textbook matrix multiplication. We also assume computing the inversion of a $n\times n$ matrix admits an $\cO(n^3)$ complexity. Interested readers may refer to \citep{papadimitriou2003computational} for more details about faster algorithms with better complexities for matrix operations.

\section{Experimental Details}

\subsection{Compute Resources}\label{sec:compute_resource}
All experiments in Section \ref{sec:exp} were conducted on a computer with an 11th Intel(R) Core(TM) i7-11370H CPU. The time and space required to run our experiments are negligible and we anticipate they can be conducted in almost all computers. 

\subsection{Experimental Details for Figure~\ref{fig:wrapfig}}\label{sec:fig3app}
In Figure~\ref{fig:wrapfig}, we show an example where the updates \eqref{eq:thetaupdatecso} may diverge first before converging eventually and finite time performance can be much worse compared to Algorithm~\ref{alg:one_sample_onlineIV}. For this experiment, we choose the model presented in \eqref{eq:exptsetting} with $d_x=d_z=1$, $\t_*=1$, $\g_*=-1$, $\rho=4$, and $\sigma_\epsilon=1$. When initialized at $\g_0=10$, and $\t_0=0$, the updates in \eqref{eq:thetaupdatecso} keeps diverging rapidly at first whereas Algorithm~\ref{alg:one_sample_onlineIV} is much more stable. So, by the end of $100,000$ iterations, while Algorithm~\ref{alg:one_sample_onlineIV} achieves an error of $\approx 10^{-5}$, \eqref{eq:thetaupdatecso} achieves $\approx 10^4$ that is worse than it was at initialization because  \eqref{eq:thetaupdatecso} has not recovered from the initial divergence phase yet. However, once \eqref{eq:thetaupdatecso} starts converging, the convergence rate of \eqref{eq:thetaupdatecso} is similar to Algorithm~\ref{alg:one_sample_onlineIV} as one can see from Figure~\ref{fig:wrapfig} (also see our discussion on the convergence of \eqref{eq:thetaupdatecso} in Section~\ref{sec:csoconvergence}).

\section{Proofs for Section~\ref{sec:two_sample}}

\subsection{Proof of Theorem \ref{thm:convergence_linear_two_sample}}
\begin{proof}
    We aim to find the optimal $\theta_*$. According to \eqref{problem:IV_parameterized}, we know
\begin{equation}\label{eq: theta_optimal}
    \EE_Z\Big[\EE_{X\mid Z}[X]\cdot \EE_{X\mid Z}[X]^\top\Big] \theta_* = \EE_Z\Big[\EE_{Y\mid Z}[Y]\cdot \EE_{X\mid Z}[X] \Big]
\end{equation}
The updates in Algorithm \ref{alg:two_sample_SGD} can be written as
\begin{align}
    \theta_{t+1} = \theta_t - \alpha_{t+1}(X_t^\top \theta_t - Y_t)X_t^\top. \notag
\end{align}
Hence we have
\begin{align}
    &\theta_{t+1} - \theta_* \notag\\
    = & \theta_t - \alpha_{t+1}\EE_Z\Big[\EE_{X\mid Z}[X]\cdot \EE_{X\mid Z}[X]^\top\Big] \theta_t + \alpha_{t+1}\EE_Z\Big[\EE_{Y\mid Z}[Y]\cdot \EE_{X\mid Z}[X] \Big] - \theta_* \notag\\
    &- \alpha_{t+1}\Big(X_t'X_t^\top - \EE_Z\Big[\EE_{X\mid Z}[X]\cdot \EE_{X\mid Z}[X]^\top\Big] \Big)\theta_t + \alpha_{t+1}\Big(Y_tX_t' -  \EE_Z\Big[\EE_{Y\mid Z}[Y]\cdot \EE_{X\mid Z}[X] \Big]\Big) \label{eq: theta_decompose_1}.
\end{align}
Now we analyze the convergence and variance separately. For the convergence part, we have
\begin{align}
    &\theta_t - \alpha_{t+1}\EE_Z\Big[\EE_{X\mid Z}[X]\cdot \EE_{X\mid Z}[X]^\top\Big] \theta_t + \alpha_{t+1}\EE_Z\Big[\EE_{Y\mid Z}[Y]\cdot \EE_{X\mid Z}[X] \Big] - \theta_* \notag\\
    = &\Big(I - \alpha_{t+1}\EE_Z\Big[\EE_{X\mid Z}[X]\cdot \EE_{X\mid Z}[X]^\top\Big]\Big)(\theta_t - \theta_*). \label{eq: convergence_eq}
\end{align}
For the variance part we have
\begin{align}
    &\EE\Big[\norm{X_t'X_t^\top - \EE_Z\Big[\EE_{X\mid Z}[X]\cdot \EE_{X\mid Z}[X]^\top\Big]}^2\Big] \notag\\
    =&\EE\Big[\norm{X_t'X_t^\top - \EE_{X\mid Z_t}[X]\cdot \EE_{X\mid Z_t}[X]^\top + \EE_{X\mid Z_t}[X]\cdot \EE_{X\mid Z_t}[X]^\top -\EE_Z\Big[\EE_{X\mid Z}[X]\cdot \EE_{X\mid Z}[X]^\top\Big]}^2\Big] \notag\\
    \leq&2\EE\Big[\norm{X_t'X_t^\top - \EE_{X\mid Z_t}[X]\cdot \EE_{X\mid Z_t}[X]^\top}^2 + \norm{\EE_{X\mid Z_t}[X]\cdot \EE_{X\mid Z_t}[X]^\top -\EE_Z\Big[\EE_{X\mid Z}[X]\cdot \EE_{X\mid Z}[X]^\top\Big]}^2\Big]\notag\\
    \leq &2C_{x}d_x^{\vartheta_1} + 2C_{xx}d_z^{\vartheta_3} =:\sigma_1^2 \label{eq: xx_var_decompose}
\end{align}
Similarly, we have
\begin{align}
    &\EE\Big[\norm{Y_tX_t' -  \EE_Z\Big[\EE_{Y\mid Z}[Y]\cdot \EE_{X\mid Z}[X] \Big]}^2\Big] \notag\\
    = & \EE\Big[\norm{Y_tX_t' - \EE_{Y\mid Z_t}[Y]\cdot \EE_{X\mid Z_t}[X]}^2 + \norm{\EE_{Y\mid Z_t}[Y]\cdot \EE_{X\mid Z_t}[X] - \EE_Z\Big[\EE_{Y\mid Z}[Y]\cdot \EE_{X\mid Z}[X] \Big]}^2\Big]\notag\\
    \leq &C_y d_x^{\vartheta_2} + C_{yx}d_z^{\vartheta_4}  =:\sigma_2^2. \label{eq: xy_var_decompose}
\end{align}
Now we know from \eqref{eq: theta_decompose_1}, \eqref{eq: convergence_eq}, \eqref{eq: xx_var_decompose}, and \eqref{eq: xy_var_decompose} that
\begin{align}\label{eq: update_recursion}
    \norm{\theta_{t+1} - \theta_*}^2 = \norm{A_t}^2 + 2\alpha_{t+1}\<A_t, B_t> + \alpha_{t+1}^2\norm{B_t}^2.
\end{align}
where
\begin{align*}
    &A_t = \Big(I - \alpha_{t+1}\EE_Z\Big[\EE_{X\mid Z}[X]\cdot \EE_{X\mid Z}[X]^\top\Big]\Big)(\theta_t - \theta_*) \\
    &B_t = -\Big(X_t'X_t^\top - \EE_Z\Big[\EE_{X\mid Z}[X]\cdot \EE_{X\mid Z}[X]^\top\Big] \Big)\theta_t + \Big(Y_tX_t' -  \EE_Z\Big[\EE_{Y\mid Z}[Y]\cdot \EE_{X\mid Z}[X] \Big]\Big).
\end{align*}
This implies 
\begin{align}
    &\EE_{\theta_{t+1} \mid \theta_t}\Big[\norm{\theta_{t+1} - \theta_*}^2\Big] \notag\\
    =&\norm{\Big(I - \alpha_{t+1}\EE_Z\Big[\EE_{X\mid Z}[X]\cdot \EE_{X\mid Z}[X]^\top\Big]\Big)(\theta_t - \theta_*)}^2 \notag\\
    &+ \alpha_{t+1}^2\EE_{X_t, X_t', Y_t, Z_t\mid \theta_t}\Big[\norm{\Big(X_t'X_t^\top - \EE_Z\Big[\EE_{X\mid Z}[X]\cdot \EE_{X\mid Z}[X]^\top\Big] \Big)\theta_t - \Big(Y_tX_t' -  \EE_Z\Big[\EE_{Y\mid Z}[Y]\cdot \EE_{X\mid Z}[X] \Big]\Big)}^2\Big] \notag\\
    \leq & (1-\alpha_{t+1}\mu)^2\norm{\theta_t - \theta_*}^2 + 3\alpha_{t+1}^2\Big(\sigma_1^2\norm{\theta_t - \theta_*}^2 + \sigma_1^2\norm{\theta_*}^2 + \sigma_2^2\norm{\theta_*}^2\Big) \notag\\
    \leq & ((1-\alpha_{t+1}\mu)^2 + 3\alpha_{t+1}^2\sigma_1^2)\norm{\theta_t - \theta_*}^2 + 3\alpha_{t+1}^2\sigma_1^2\norm{\theta_*}^2 + 3\alpha_{t+1}^2\sigma_2^2\norm{\theta_*}^2, \label{ineq: theta_exp}
\end{align}
where the first inequality uses Cauchy-Schwarz inequality, the definition of $\sigma_1, \sigma_2$ and Assumption \ref{aspt: var_general}. Choosing $\alpha_{t+1}$ such that
\begin{equation*}
    ((1-\alpha_{t+1}\mu)^2 + 3\alpha_{t+1}^2\sigma_1^2)\leq 1-\alpha_{t+1}\mu \Leftrightarrow \alpha \leq \frac{\mu}{\mu^2 + 3\sigma_1^2}
\end{equation*}
and taking expectation on both sides of \eqref{ineq: theta_exp}, we have 
\begin{equation*}
    \EE\Big[\norm{\theta_{t+1} - \theta_*}^2\Big] \leq (1-\alpha_{t+1}\mu)\EE\Big[\norm{\theta_t - \theta_*}^2\Big] + 3\alpha_{t+1}^2\sigma_1^2\norm{\theta_*}^2 + 3\alpha_{t+1}^2\sigma_2^2\norm{\theta_*}^2.
\end{equation*}
Now, we use the following result.
\begin{lemma}\label{lem: recursive_ineq}
    Suppose we have three sequences $\{a_t\}_{t=0}^{\infty}, \{b_t\}_{t=0}^{\infty}, \{r_t\}_{t=0}^{\infty}$ satisfying
    \begin{align}\label{ineq: recursion}
        a_{t+1}\leq r_ta_t + b_t, r_t > 0
    \end{align}
    for any $t\geq 0$. Define $R_{t+1} = \prod_{i=0}^{t}r_i$, we have
    \begin{align*}
        a_{t+1}\leq R_{t+1}a_0 + \sum_{i=0}^{t}\frac{R_{t+1}b_i}{R_{i+1}}.
    \end{align*}
\end{lemma}

By Lemma \ref{lem: recursive_ineq}, we know
\begin{align*}
    \EE\Big[\norm{\theta_{t+1} - \theta_*}^2\Big] \leq \prod_{i=0}^{t}(1-\alpha_i\mu)\EE\Big[\norm{\theta_0 - \theta_*}^2\Big] + (3\sigma_1^2\norm{\theta_*}^2 + 3\sigma_2^2\norm{\theta_*}^2)\sum_{i=0}^{t}\alpha_i^2\prod_{j=i+1}^{t}(1-\alpha_j\mu).
\end{align*}
Now if we set $\alpha_i = \alpha$, we know
\begin{align*}
    \EE\Big[\norm{\theta_t - \theta_*}^2\Big]\leq &(1-\alpha\mu)^t\EE\Big[\norm{\theta_0 - \theta_*}^2\Big] + \alpha^2\Big(\sum_{i=0}^{t}(1-\alpha\mu)^i\Big)(3\sigma_1^2\norm{\theta_*}^2 + 3\sigma_2^2\norm{\theta_*}^2) \\
    \leq &e^{-t\alpha\mu}\EE\Big[\norm{\theta_0 - \theta_*}^2\Big] + \frac{\alpha}{\mu}(3\sigma_1^2\norm{\theta_*}^2 + 3\sigma_2^2\norm{\theta_*}^2)
\end{align*}
Choosing $\alpha, T$ such that $\alpha = \frac{\log T}{\mu T}\leq \frac{\mu}{\mu^2 + 3\sigma_1^2}$, we know 
\begin{align*}
    \EE\Big[\norm{\theta_T - \theta_*}^2\Big]\leq \frac{\EE\Big[\norm{\theta_0 - \theta_*}^2\Big]}{T} + \frac{3\norm{\theta_*}^2(\sigma_1^2 + \sigma_2^2)\log T}{\mu^2T}.
\end{align*}
\end{proof}

\subsection{Proof of Lemma~\ref{lem: recursive_ineq}}
\label{appendix:technical}

\begin{proof}
    We notice from \eqref{ineq: recursion} that for any $0\leq i\leq t$, we have
    \begin{align*}
        \frac{a_{i+1}}{R_{i+1}}\leq \frac{a_i}{R_i} + \frac{b_i}{R_{i+1}}.
    \end{align*}
    Taking summation on both sides, we have
    \begin{align*}
        \frac{a_{t+1}}{R_{t+1}}\leq \frac{a_0}{R_0} + \sum_{i=0}^{t}\frac{b_i}{R_{i+1}}
    \end{align*}
    which completes the proof by multiplying $R_{t+1}$ on both sides.
\end{proof}



\subsection{Proof of Lemma \ref{lem:bdd_var_linear}}\label{appendix:lemma1_proof}
\begin{proof}
We first notice that Assumption \ref{aspt: scvx} holds since
\begin{align*}
    \EE_Z\Big[\EE_{X\mid Z}[X]\cdot \EE_{X\mid Z}[X]^\top\Big] = \EE_Z\Big[\phi(\gamma_*^\top Z)\cdot \phi(\gamma_*^\top Z)^\top\Big] \succeq \mu I.
\end{align*}
For \eqref{ineq: px} and \eqref{ineq: py}, we have
\begin{align}
    &\EE\Big[\norm{X' X^\top - \EE_{X|Z}[X]\EE_{X|Z}[X]^\top}^2\Big] \notag \\
    = &\EE\Big[\norm{\epsilon_2' \phi(\gamma_*^\top Z)^\top + \phi(\gamma_*^\top Z)\epsilon_2^\top + \epsilon_2'\epsilon_2^\top}^2\Big] \notag\\
    \leq &3\EE\Big[\norm{\epsilon_2' \phi(\gamma_*^\top Z)^\top}^2 + \norm{\phi(\gamma_*^\top Z)\epsilon_2^\top}^2 + \norm{\epsilon_2'\epsilon_2^\top}^2\Big] \notag \\
    = &3\EE\Big[\norm{\phi(\gamma_*^\top Z)\epsilon_2'^\top\epsilon_2' \phi(\gamma_*^\top Z)^\top} + \norm{\phi(\gamma_*^\top Z)\epsilon_2^\top\epsilon_2 \phi(\gamma_*^\top Z)^\top} + |\epsilon_2^\top \epsilon_2'|^2\Big] = \cO(d_x^2), \label{ineq: special_xx_nonlinear}
\end{align}
and 
\begin{align*}
    &\EE\Big[\norm{YX - \EE_{Y|Z}[Y]\EE_{X|Z}[X]}^2\Big] \\
    = &\EE\Big[\norm{X'X^\top\theta_* + \epsilon_1 X - \EE_{X|Z}[X]\EE_{X|Z}[X]^\top\theta_*}^2\Big] \\
    \leq &2\EE\Big[\norm{X'X^\top\theta_* - \EE_{X|Z}[X]\EE_{X|Z}[X]^\top\theta_*}^2\Big] + 2\EE\Big[\epsilon_1^2\norm{\phi(\gamma_*^\top Z) + \epsilon_2}^2\Big] \\
    = &\cO(\norm{\theta_*}^2\sigma_{\epsilon_2}^2d_x^2 + \sigma_{\epsilon_1}^2d_x + \sigma_{\epsilon_1, \epsilon_2}^2d_x),
\end{align*}
where the first inequality uses Cauchy-Schwarz inequality, and the second equality uses \eqref{eq: special_nonlinear_setup}, \eqref{ineq: special_nonlinear_setup_ineq} and \eqref{ineq: special_xx_nonlinear}. For \eqref{ineq: pxx} we have
\begin{align*}
    &\EE\Big[\norm{\EE_{X\mid Z}[X]\cdot \EE_{X\mid Z}[X]^\top -\EE_Z\Big[\EE_{X\mid Z}[X]\cdot \EE_{X\mid Z}[X]^\top\Big]}^2\Big] \\
    = &\EE\Big[\norm{\phi(\gamma_*^\top Z)\phi(\gamma_*^\top Z)^\top - \EE\left[\phi(\gamma_*^\top Z)\phi(\gamma_*^\top Z)^\top\right]}^2\Big] = \cO(d_z) 
\end{align*}
where the last equality uses \eqref{ineq: special_nonlinear_setup_ineq}. Using the above conclusion in \eqref{ineq: pyx}, we have
\begin{align*}
    &\EE\Big[\norm{\EE_{Y\mid Z}[Y]\cdot \EE_{X\mid Z}[X] -\EE_Z\Big[\EE_{Y\mid Z}[Y]\cdot \EE_{X\mid Z}[X]\Big]}^2\Big] \\
    =&\EE\Big[\norm{\EE_{X\mid Z}[X]\cdot \EE_{X\mid Z}[X]^\top\theta_* -\EE_Z\Big[\EE_{X\mid Z}[X]\cdot \EE_{X\mid Z}[X]^\top\theta_*\Big]}^2\Big] = \cO(\norm{\theta_*}^2d_z).
\end{align*}
\end{proof}
%
%
%
%
\section{Proofs for Section~\ref{sec:one_sample}}
\subsection{Proof of Theorem~\ref{th:mainthetaconv}}\label{sec:pfmainthetaconv} 
\begin{proof}[Proof of Theorem~\ref{th:mainthetaconv} ]\label{pf:mainthetaconv} 
Recall that $\xi_{Z_t}$, and $\xi_{Z_tY_t}$ are the i.i.d. noise sequences 
\begin{align*}
   \xi_{Z_t}&=\Sigma_Z-Z_tZ_t^\top,\\
   \xi_{Z_tY_t}&=\Sigma_{ZY}-Z_tY_t.
\end{align*} 
Note $\g_*$, and $\t_*$ can be written as $
    \g_*=\Sigma_Z^{-1}\Sigma_{ZX}\in\mathbb{R}^{d_z\times d_x}$, and $\t_*=\left(\g_*^\top\Sigma_z\gamma_*\right)^{-1}\g_*^\top \Sigma_{ZY} \in\mathbb{R}^{d_x}
$ which we are going to use throughout the proof.

To quantify the bias, we use the following bound on $\EE\left[\norm{\gt-\g_*}_{2}^k\right]$, $k=1,2,4$, proved in Lemma 3.2 of \cite{chen2020statistical}.

\begin{lemma}\label{lm:gtconvrate}
Suppose Assumption~\ref{aspt: independence}, and Assumption~\ref{as:sigmazposdef} hold. Then we have
\begin{align*}
    \EE\left[\norm{\gt-\g_*}^k\right]=O\left(\sqrt{d_z^k\bet^k}\right)\quad \text{for} \quad k=1,2,4. \numberthis\label{eq:gtconvergencerate}
\end{align*}
\end{lemma} 
 We proceed by noting that if $\g_*$, $\Sigma_Z$, and $\Sigma_{ZY}$ were known beforehand, one could use the following deterministic gradient updates to obtain $\t_*$. 
 \begin{align*}
 \tilde{\t}_{t+1}=\tilde{\t}_t-\atone\g_*^\top \left(\Sigma_Z\g_*\tilde{\t}_t-\Sigma_{ZY}\right).\numberthis\label{eq:tildethetaupdate}
 \end{align*}
\begin{lemma}\label{lm:tildethetaconvrate}
Let Assumption~\ref{aspt: scvx} be true. Then, choosing $\eta_{k}=O(k^{-a})$ with $1/2<a<1$, we have
$
    \norm{\tilde{\t}_t-\t_*}=O\left(\exp(-t^{1-a})\right).
$
\end{lemma}

Define the sequence $\delt\coloneqq \tt-\tilde{\t}_t$. We will establish the convergence rate of $\expec{\lVert\delt\rVert_2^2}{}$. From \eqref{eq:thetaupdate}, and \eqref{eq:tildethetaupdate}, we have the following expansion of $\deltone$. 
\begin{align*}
    \deltone=Q_t\delt +\atone D_t\tt+\atone (\gt-\g_*)^\top \Sigma_{ZY}
    -\atone \gt^\top \xi_{Z_tY_t}+\atone \gt^\top\xi_{Z_t}\gt\tt,\numberthis\label{eq:delupdate}
\end{align*}
where 
\begin{align*}
Q_t\coloneqq & (I-\atone \g_*^\top \Sigma_Z\g_*),\\
D_t\coloneqq& \g_*^\top \Sigma_Z\g_*-\gt^\top \Sigma_Z\gt.
\end{align*}
First we will establish an intermediate bound on $\expec{\normtt{\delt}}{}$. To do so, we will need the following result which shows that $\expec{\norm{\tt-\t_*}_2^4}{}$ is bounded for all $t$ which we prove in Section~\ref{pf:thetabounded}.
\begin{lemma}[\textbf{Boundedness of fourth moment of $\norm{\tt-\t_*}$}]\label{lm:thetabounded}
    Let the conditions in Theorem~\ref{th:mainthetaconv} be true. Then, choosing $\at,\bet$ such that $\at\leq d_z^{-4\varkappa-\vartheta/2}$, and $\sum_{t=1}^\infty(\atsqr+\at\sqrt{\bet})<\infty$, we have $\expec{\norm{\tt-\t_*}_2^4}{}$ is bounded by some constant $M>0$.
\end{lemma}
\begin{lemma}[\textbf{Intermediate bound on $\EE[\lVert\delta_t\rVert_2^2]$}]\label{lm:deltintermedbound}
Let the conditions in Theorem~\ref{th:mainthetaconv} be true. We have the following intermediate bound on $\expec{\normtt{\delt}}{}$:
    \begin{align*}
        \expec{\normtt{\delt}}{}=O\left(\bet d_z^{1+2\varkappa}+\atone d_z^{4\varkappa+\vartheta/2}
    +\sqrt{d_z\bet}\right).\numberthis\label{eq:deltintermedbound}
    \end{align*}
\end{lemma}
\begin{proof}[Proof of Lemma~\ref{lm:deltintermedbound}]\label{pf:deltintermedbound}
Recall the update for $\deltone$ obtained in \eqref{eq:delupdate}.
\begin{align*}
    \deltone=Q_t\delt +\atone D_t\tt+\atone (\gt-\g_*)^\top \Sigma_{ZY}
    -\atone \gt^\top \xi_{Z_tY_t}+\atone \gt^\top\xi_{Z_t}\gt\tt.
\end{align*}
Then,
\begin{align}\label{eq:deltasqrexpandintermed}
\begin{aligned}
    \norm{\deltone}_2^2
    =&\delt^\top Q_t^2\delt+\atonesqr\normtt{D_t\tt+ (\gt-\g_*)^\top \Sigma_{ZY}
    - \gt^\top \xi_{Z_tY_t}+ \gt^\top\xi_{Z_t}\gt\tt}\\&+2\atone \delt^\top Q_t\left(D_t\tt+(\gt-\g_*)^\top \Sigma_{ZY}\right)\\
    &+2\atone \delt^\top Q_t\left( \gt^\top\xi_{Z_t}\gt\tt-\gt^\top \xi_{Z_tY_t}\right). 
\end{aligned}    
\end{align}
Then, choosing $\alpha_1(\norm{\g_*}_2\lambda_Z)^2<1 $, using Young's inequality and Assumption~\ref{aspt: independence}, from \eqref{eq:deltasqrexpandintermed} we get, 
\begin{align*}
    \expec{\norm{\deltone}_2^2}{t}
    \leq &(1-\atone\mu)\normtt{\delt}+4\atonesqr\left(\normtt{D_t\tt}+\normtt{ (\gt-\g_*)^\top \Sigma_{ZY}}\right)\\
    &+4\atonesqr \left(\norm{\gt}_2^2\expec{\norm{ \xi_{Z_tY_t}}_2^2}{}+\norm{\gt}_2^4\expec{\normtt{ \xi_{Z_t}}}{}\normtt{\tt}\right)\\
    &+2\atone \delt^\top Q_t\left(D_t\tt+(\gt-\g_*)^\top \Sigma_{ZY}\right)\\
    \lesssim &(1-\atone\mu)\normtt{\delt}+4\atonesqr\left(\normtt{D_t}\normtt{\tt}+\normtt{ (\gt-\g_*)^\top \Sigma_{ZY}}\right)\\
    &+4C\atonesqr \left(d_z^{2\varkappa+\vartheta/2}+d_z^{4\varkappa+\vartheta/2}\normtt{\tt}\right)
    +\\
    &2\atone \delt^\top Q_t\left(D_t\tt+(\gt-\g_*)^\top \Sigma_{ZY}\right),
\end{align*}
where the last inequality follows by Assumption~\ref{as:boundedcovarnoisevar},and Assumption~\ref{as:gtbounded}. 

Now, taking expectation on both sides, we obtain
\begin{align}\label{eq:expecdeltintermed}
\begin{aligned}
    \expec{\norm{\deltone}_2^2}{}\lesssim &(1-\atone\mu)\expec{\normtt{\delt}}{}+4\atonesqr\left(\expec{\normtt{D_t}\normtt{\tt}}{}+\expec{\normtt{ (\gt-\g_*)^\top \Sigma_{ZY}}}{}\right)\\
    &+4C\atonesqr \left(d_z^{2\varkappa+\vartheta/2}+d_z^{4\varkappa+\vartheta/2}\expec{\normtt{\tt}}{}\right)\\
    &+2\atone\left(\expec{|\delt^\top Q_tD_t\tt|}{}+\expec{|\delt^\top Q_t(\gt-\g_*)^\top \Sigma_{ZY}|}{}\right).
\end{aligned}
\end{align}
Now, the  following bounds are true:
\begin{enumerate}[leftmargin=0.1in]
    \item We have that 
    \begin{align*}
       \atonesqr \expec{\normtt{D_t}\normtt{\tt}}{}\leq \atonesqr\sqrt{\expec{\norm{D_t}_2^4}{}\expec{\norm{\tt}_2^4}{}}\lesssim d_z^{1+2\varkappa}\atonesqr\bet,\numberthis\label{eq:deltoneterm1}
    \end{align*}
    where the first inequality follows by Cauchy-Schwarz inequality, the second inequality follows by \eqref{eq:Dtnormbound}, and Lemma~\ref{lm:thetabounded}.
    \item Using $\Sigma_{ZY}=O(1)$, and Lemma~\ref{lm:gtconvrate}, we get
    \begin{align*}
       \atonesqr \expec{\normtt{ (\gt-\g_*)^\top \Sigma_{ZY}}}{}\lesssim d_z\bet\atonesqr.\numberthis\label{eq:deltoneterm1b}
    \end{align*}
    \item We have that
    \begin{align*}
        \atone\expec{|\delt^\top Q_tD_t\tt|}{}\leq & \atone\expec{\norm{\delt}_2\norm{Q_t}_2\norm{D_t}_2\norm{\tt}_2}{}\\
        \leq& \frac{\atone\mu}{16}\expec{\normtt{\delt}}{}+\frac{4\atone}{\mu}\sqrt{\expec{\norm{D_t}_2^4}{}\expec{\norm{\tt}_2^4}{}}\\
        \lesssim & \frac{\atone\mu}{16}\expec{\normtt{\delt}}{}+\frac{4 d_z^{1+2\varkappa}\atone\bet}{\mu},\numberthis\label{eq:delqdthetabound}
    \end{align*}
     where the first inequality follows by Hölder's inequality, the second inequality follows by Young's inequality, Cauchy-Schwarz inequality, and $\norm{Q_t}_2< 1$,   and the third inequality follows by \eqref{eq:Dtnormbound}, and Lemma~\ref{lm:thetabounded}.
    \item Using $\norm{Q_t}_2< 1$, $\norm{\Sigma_{ZY}}_2=O(1)$, Cauchy-Schwarz inequality, and Lemma~\ref{lm:gtconvrate}, we get, 
    \begin{align*}
        &\atone\expec{|\delt^\top Q_t(\gt-\g_*)^\top \Sigma_{ZY}|}{}\\
        \lesssim &\atone\expec{\norm{\delt}_2\norm{\gt-\g_*}_2}{}\\
        \leq &\atone\sqrt{\expec{\norm{\delt}_2^2}{}\expec{\norm{\gt-\g_*}_2^2}{}}\\
        \leq & \frac{\sqrt{d_z\bet}\atone}{2}+\frac{\sqrt{d_z\bet}\atone\expec{\norm{\delt}_2^2}{}}{2}.\numberthis\label{eq:delqgtgstarSigmabound}
    \end{align*} 
\end{enumerate}
Combining \eqref{eq:expecdeltintermed}, \eqref{eq:deltoneterm1}, \eqref{eq:deltoneterm1b}, \eqref{eq:delqdthetabound}, \eqref{eq:delqgtgstarSigmabound}, and Lemma~\ref{lm:thetabounded}, we have 
\begin{align*}
    &\expec{\norm{\deltone}_2^2}{}\\
    \lesssim &(1-\atone\mu)\expec{\normtt{\delt}}{}+4\atonesqr\bet d_z^{1+2\varkappa}
    +4C\atonesqr d_z^{4\varkappa+\vartheta/2}\\
    &+2\atone\left(\mu\expec{\normtt{\delt}}{}/16+4 d_z^{1+2\varkappa}\bet/\mu+\sqrt{d_z\bet}/2+\sqrt{d_z\bet}\expec{\norm{\delt}_2^2}{}/2\right)\numberthis\label{eq:deltrecursionintermed}\\
    \lesssim &(1-7\mu\atone/8+\atone\sqrt{d_z\bet})\expec{\normtt{\delt}}{}+(8\atone\bet d_z^{1+2\varkappa}/\mu+4C\atonesqr d_z^{4\varkappa+\vartheta/2})\\
    &+\atone\sqrt{d_z\bet}\\
    \lesssim &(1-3\mu\atone/4)\expec{\normtt{\delt}}{}+(8\atone\bet d_z^{1+2\varkappa}/\mu+4C\atonesqr d_z^{4\varkappa+\vartheta/2})
    +\atone\sqrt{d_z\bet}. \numberthis\label{eq:deltrecursion}
\end{align*}
In the above, the third inequality follows by choosing $\bet\leq \mu^2/(64d_z)$, and $\atone\sqrt{d_z\bet}<1$. Then, from \eqref{eq:deltrecursion}, we have
\begin{align*}
    \expec{\norm{\delt}_2^2}{}= O\left(\bet d_z^{1+2\varkappa}+\atone d_z^{4\varkappa+\vartheta/2}
    +\sqrt{d_z\bet}\right).
\end{align*}
\end{proof}
Coming back to the proof of Theorem~\ref{th:mainthetaconv}, observe that, we can sharpen the bound in \eqref{eq:delqgtgstarSigmabound} using Lemma~\ref{lm:deltintermedbound} which allows us to avoid the use of Young's inequality. This leads to the following improved version of the recursion in \eqref{eq:deltrecursion} using which we can improve the term $\sqrt{d_z\bet}$ in \eqref{eq:deltintermedbound} as follows: 
\begin{align*}
       \expec{\norm{\deltone}_2^2}{}
     \lesssim &(1-7\mu\atone/8)\expec{\normtt{\delt}}{}\\
     &+\atone O\left(\bet d_z^{1+2\varkappa}+\atone d_z^{4\varkappa+\vartheta/2}+\sqrt{\atone\bet}d_z^{1/2+2\varkappa+\vartheta/4}+(\bet d_z)^{3/4}\right)\\
    =&O\left(\bet d_z^{1+2\varkappa}+\atone d_z^{4\varkappa+\vartheta/2}+\sqrt{\atone\bet}d_z^{1/2+2\varkappa+\vartheta/4}+(\bet d_z)^{3/4}\right). 
     \end{align*}
In fact, this trick can be used repeatedly to sharpen the bound even further  as shown in Lemma~\ref{lm:improvedfinaldeltrate}.
\begin{lemma}[\textbf{Final improved bound on $\EE[\lVert\delta_t\rVert_2^2]$}]\label{lm:improvedfinaldeltrate}
Let the conditions in Theorem~\ref{th:mainthetaconv} be true. Then using Lemma~\ref{lm:deltintermedbound}, we have, 
    \begin{align*}
       &\expec{\norm{\deltone}_2^2}{}\\
       \lesssim  & O\lrbrac{(d_z\bet)^{1-2^{-r-1}}+\sum_{i=0}^r\left(\alpha_{t+1}^{2^{-i}}\bet^{1-2^{-i}} d_z^{1+(4\varkappa+\vartheta/2-1)2^{-i}}+\bet (1+\alpha_{t+1}^{2^{-i}}) d_z^{1+2^{1-i}\varkappa}\right)},
    \end{align*}
    where $r$ is any non-negative integer. 
\end{lemma}
\begin{proof}[Proof of Lemma~\ref{lm:improvedfinaldeltrate}]\label{pf:improvedfinaldeltrate}
   If we have $$\expec{\normtt{\delt}}{}=O\left(\atone d_z^{4\varkappa+\vartheta/2}+\bet d_z^{1+2\varkappa}
    +\sqrt{d_z\bet}\right),$$ 
    then from \eqref{eq:delqgtgstarSigmabound}, we have, 
   \begin{align*}
        \expec{|\delt^\top Q_t(\gt-\g_*)^\top \Sigma_{ZY}|}{}\lesssim &\sqrt{\expec{\norm{\delt}_2^2}{}\expec{\norm{\gt-\g_*}_2^2}{}}\\
        = &O\left(\sqrt{\atone\bet}d_z^{1/2+2\varkappa+\vartheta/4}+\bet d_z^{1+\varkappa}+(d_z\bet)^{3/4}\right).\numberthis\label{eq:delQgtSigmafinalbound}
    \end{align*} 
   Then, similar to \eqref{eq:deltrecursionintermed}, we have, 
   \begin{align*}
       &\expec{\norm{\deltone}_2^2}{}\\
       \lesssim &(1-\atone\mu)\expec{\normtt{\delt}}{}+4\atonesqr\bet d_z^{1+2\varkappa}
    +4C\atonesqr d_z^{4\varkappa+\vartheta/2}\\
    &+2\atone\left(\mu\expec{\normtt{\delt}}{}/16+4 d_z^{1+2\varkappa}\bet/\mu+\sqrt{\atone\bet}d_z^{1/2+2\varkappa+\vartheta/4}+\bet d_z^{1+\varkappa}+(d_z\bet)^{3/4}\right)\\
     \lesssim &(1-7\mu\atone/8)\expec{\normtt{\delt}}{}\\
     &+\atone O\lrbrac{(d_z\bet)^{3/4}+\sum_{i=0}^1\left(\alpha_{t+1}^{2^{-i}}\bet^{1-2^{-i}} d_z^{1+(4\varkappa+\vartheta/2-1)2^{-i}}+\bet (1+\alpha_{t+1}^{2^{-i}}) d_z^{1+2^{1-i}\varkappa}\right)}\\
    =&O\lrbrac{(d_z\bet)^{3/4}+\sum_{i=0}^1\left(\alpha_{t+1}^{2^{-i}}\bet^{1-2^{-i}} d_z^{1+(4\varkappa+\vartheta/2-1)2^{-i}}+\bet (1+\alpha_{t+1}^{2^{-i}}) d_z^{1+2^{1-i}\varkappa}\right)}. 
     \end{align*}
     Now if we repeat this step $r$ number of times (where $r$ is to be set later), by progressive sharpening we get the following bound.
     \begin{align*}
       &\expec{\norm{\deltone}_2^2}{}\\
       \lesssim & O\lrbrac{(d_z\bet)^{1-2^{-r-1}}+\sum_{i=0}^r\left(\alpha_{t+1}^{2^{-i}}\bet^{1-2^{-i}} d_z^{1+(4\varkappa+\vartheta/2-1)2^{-i}}+\bet (1+\alpha_{t+1}^{2^{-i}}) d_z^{1+2^{1-i}\varkappa}\right)}.
    \end{align*}
\end{proof}
Coming back to the proof of Theorem~\ref{th:mainthetaconv}, we have that by combining Lemma~\ref{lm:tildethetaconvrate}, and Lemma~\ref{lm:improvedfinaldeltrate}, 
\begin{align*}
    &\expec{\normtt{\tt-\t_*}}{}\leq  2\expec{\normtt{\delt}}{}+2\expec{\normtt{\tilde{\theta}_t-\t_*}}{}\\
    =&O\lrbrac{(d_z\bet)^{1-2^{-r-1}}+\sum_{i=0}^r\left(\alpha_{t+1}^{2^{-i}}\bet^{1-2^{-i}} d_z^{1+(4\varkappa+\vartheta/2-1)2^{-i}}+\bet (1+\alpha_{t+1}^{2^{-i}}) d_z^{1+2^{1-i}\varkappa}\right)}.\numberthis\label{eq:thetaconvgeneralalphabetabound}
\end{align*}
Now, in \eqref{eq:thetaconvgeneralalphabetabound}, for some arbitrarily small number $\iota>0$, choosing $$\at=\min(0.5d_z^{-4\varkappa-\vartheta/2}\lambda_Z^{-1}C_\g^{-2},0.5(\norm{\g_*}_2\lambda_Z)^{-2})t^{-1+\iota/2}, \qquad \bet= \mu^2 d_z^{-1-2\varkappa}t^{-1+\iota/2}/128,$$
and setting $r=\lceil\log_2{((\iota/2)^{-1}-1)}-1\rceil$ we get, 
\begin{align*}
    \expec{\normtt{\tt-\t_*}}{}=O\lrbrac{\max\lrbrac{t^{-1+\iota},t^{-1+\iota/2}\log((\iota/2)^{-1}-1)}}.
\end{align*}
\end{proof}

\subsection{Proof of Lemma~\ref{lm:thetabounded}}
\begin{proof}\label{pf:thetabounded}
Using the form of $\t_*$, from \eqref{eq:thetaupdate} we get, 
\begin{align*}
    \ttone-\t_*=&\hQt(\tt-\t_*)+\atone(\gt-\g_*)^\top\Sigma_{ZY}+\atone D_t\t_*+\atone\gt^\top\xi_{Z_t}\gt(\tt-\t_*)\\
    &+\atone\gt^\top\xi_{Z_t}\gt\t_*+\atone\gt^\top\xi_{Z_tY_t}. \numberthis\label{eq:thetaupdateexpanded}
\end{align*}
where $\hQt\coloneqq\lrbrac{I-\atone\gt^\top\Sigma_Z\gt}=Q_t+\atone D_t$. Recall that $D_t= \g_*^\top \Sigma_Z\g_*-\gt^\top \Sigma_Z\gt$. By Assumption~\ref{as:gtbounded}, we have the following bound on $\norm{D_t}_2$.
\begin{align*}
    \norm{D_t}_2=O(\lambda_Z C_\g^2 d_z^{2\varkappa}).\numberthis\label{eq:Dtbound}
\end{align*}
We have the following bound on $\expec{\norm{D_t}_2^4}{}$ by Lemma~\ref{lm:gtconvrate}.
\begin{align*}
    \expec{\norm{D_t}_2^4}{}=\expec{\norm{(\g_*-\gt)^\top \Sigma_Z\g_*+\gt^\top \Sigma_Z(\g_*-\gt)}_2^4}{}=O(d_z^{2+4\varkappa}\betsqr).\numberthis\label{eq:Dtnormbound}
\end{align*}
From \eqref{eq:thetaupdateexpanded}, we have
\begin{align*}
    \norm{\ttone-\t_*}_2^2
    \leq  & (\tt-\t_*)^\top\hQt^2(\tt-\t_*) +3\atonesqr\norm{\gt^\top\xi_{Z_t}\gt(\tt-\t_*)}_2^2\\
    &+2\atone (\tt-\t_*)^\top\hQt(\gt-\g_*)^\top \Sigma_{ZY}\\ 
    &+2\atone (\tt-\t_*)^\top\hQt D_t\t_*+A_{1,t}+A_{2,t},\numberthis\label{eq:A12introduced}
\end{align*}
where 
\begin{align*}
A_{1,t}=&\alpha_{t+1}^2\big(\norm{(\gt-\g_*)^\top \Sigma_{ZY}}_2^2+\norm{D_t\t_*}_2^2\\
&+2 \Sigma_{ZY}^\top (\gt-\g_*)D_t\t_*+3\norm{\gt^\top\xi_{Z_t}\gt\t_*}_2^2+3\norm{\gt^\top\xi_{Z_tY_t}}_2^2\big), \numberthis\label{eq:A1tdef}
\end{align*}
and 
\begin{align*}
A_{2,t}=&2\atone (\hQt(\tt-\t_*)+\atone(\gt-\g_*)^\top\Sigma_{ZY}\\
&+\atone D_t\t_*)^\top(\gt^\top\xi_{Z_t}\gt(\tt-\t_*)
    +\gt^\top\xi_{Z_t}\gt\t_*+\gt^\top\xi_{Z_tY_t}).
\end{align*}
Define 
\begin{align*}
A_{3,t}\coloneqq&3\atonesqr\norm{\gt^\top\xi_{Z_t}\gt(\tt-\t_*)}_2^2+2\atone (\tt-\t_*)^\top\hQt(\gt-\g_*)^\top \Sigma_{ZY}\\
&+2\atone (\tt-\t_*)^\top\hQt D_t\t_*+A_{1,t}+A_{2,t}.\numberthis\label{eq:A3tdef}
\end{align*}
Then, choosing $C_\g^2d_z^{2\varkappa}\lambda_Z\atone<1$, which ensures $\lVert \hQt\rVert\leq 1$, we have
    \begin{align*}
    \norm{\ttone-\t_*}_2^4
    \leq  \norm{\tt-\t_*}_2^4+2(\tt-\t_*)^\top\hQt^2(\tt-\t_*) A_{3,t}+A_{3,t}^2. \numberthis\label{eq:theta4recintermed}
\end{align*}
\vspace{0.1in}
We now have the following bounds: 
    \begin{enumerate}[leftmargin=0.1in]
        \item Using Assumption~\ref{as:boundedcovarnoisevar}, and Assumption~\ref{as:gtbounded},
        \begin{align}
            \alpha_{t+1}^4\expec{\norm{\gt^\top\xi_{Z_t}\gt(\tt-\t_*)}_2^4}{}\lesssim d_z^{8\varkappa+\sigmar} \alpha_{t+1}^4\expec{\norm{\tt-\t_*}_2^4}{}.\label{eq:A31bound}
        \end{align}
        \item We have that
        \begin{align*}
            &\expec{((\tt-\t_*)^\top\hQt(\gt-\g_*)^\top \Sigma_{ZY})^2}{} \\
            &\lesssim \expec{\normtt{\tt-\t_*}\normtt{\gt-\g_*}}{}\\
             &\leq \sqrt{\expec{\norm{\tt-\t_*}_2^4}{}\expec{\norm{\gt-\g_*}_2^4}{}}\\
             & \leq d_z\bet\left(1+\expec{\norm{\tt-\t_*}_2^4}{}\right)/2,\numberthis\label{eq:thetaqgtSigmaboud}
        \end{align*}
        where, the first inequality follows by $\lVert\hQt\rVert_2=O(1)$, and $\lVert\Sigma_{ZY}\rVert_2=O(1)$. The second inequality follows by Cauchy-Schwarz inequality. The last inequality follows by $\sqrt{ab}\leq (a+b)/2$, and Lemma~\ref{lm:gtconvrate}.
        \item We have that
        \begin{align*}
            &\expec{((\tt-\t_*)^\top\hQt D_t\t_*)^2}{}\\
            &\lesssim \expec{\norm{\tt-\t_*}_2^2\norm{ D_t}_2^2}{}\\
            &\leq \sqrt{\expec{\norm{\tt-\t_*}_2^4}{}\expec{\norm{D_t}_2^4}{}}\\
            &\lesssim d_z^{1+2\varkappa}\bet\left(1+\expec{\norm{\tt-\t_*}_2^4}{}\right)/2, \numberthis\label{eq:A33bound}
        \end{align*}
        where, the first inequality follows by $\lVert\hQt\rVert_2=O(1)$, and $\lVert\t_*\rVert_2=O(1)$. The second inequality follows by Cauchy-Schwarz inequality. The last inequality follows by $\sqrt{ab}\leq (a+b)/2$, and \eqref{eq:Dtnormbound}.
        \item Using Assumption~\ref{as:boundedcovarnoisevar}, Assumption~\ref{as:gtbounded}, \eqref{eq:Dtnormbound}, and Lemma~\ref{lm:gtconvrate}, we have 
        \begin{align}
            \expec{A_{1,t}^2}{}=O\lrbrac{d_z^{8\varkappa + \vartheta}\alpha_{t+1}^4}.\label{eq:A1squarebound}
        \end{align}.
        \item Using Young's inequality, Assumption~\ref{as:boundedcovarnoisevar}, Assumption~\ref{as:gtbounded}, Lemma~\ref{lm:gtconvrate}, $\norm{\Sigma_{ZY}}_2=O(1)$, $\norm{\t_*}_2=O(1)$, and \eqref{eq:Dtnormbound}, we have
        \begin{align*}
            \expec{A_{2,t}^2}{}\leq& 2\alpha_{t+1}^2\expec{\norm{\hQt(\tt-\t_*)+\atone(\gt-\g_*)^\top\Sigma_{ZY}+\atone D_t\t_*}_2^4}{}\\
            &+2\alpha_{t+1}^2\expec{\norm{\gt^\top\xi_{Z_t}\gt(\tt-\t_*)+\gt^\top\xi_{Z_t}\gt\t_*+\gt^\top\xi_{Z_tY_t}}_2^4}{}\\
            \lesssim &\alpha_{t+1}^2d_z^{8\varkappa+\vartheta}(1+\expec{\norm{\tt-\t_*}_2^4}{}).\numberthis\label{eq:A35bound}
        \end{align*}
          \item  Using $\lVert\hQt\rVert_2=O(1)$, Assumption~\ref{as:boundedcovarnoisevar}, and Assumption~\ref{as:gtbounded},
        \begin{align*}
          \atonesqr\expec{(\tt-\t_*)^\top\hQt^2(\tt-\t_*) \norm{\gt^\top\xi_{Z_t}\gt(\tt-\t_*)}_2^2}{}\lesssim & \atonesqr d_z^{4\varkappa+\vartheta/2}\expec{\norm{\tt-\t_*}_2^4}{}.  \numberthis\label{eq:A31interact}
        \end{align*}
        \item We have that
        \begin{align*}
            &\atone\expec{|(\tt-\t_*)^\top\hQt^2(\tt-\t_*) (\tt-\t_*)^\top\hQt(\gt-\g_*)^\top \Sigma_{ZY}|}{}\\
            \lesssim & \atone\expec{\norm{\tt-\t_*}_2^3\norm{\gt-\g_*}_2}{}\\
            \leq &\atone\left(\expec{\norm{\tt-\t_*}_2^4}{}\right)^{3/4}\left(\expec{\norm{\gt-\g_*}_2^4}{}\right)^{1/4}\\
            \leq & \atone\sqrt{d_z\bet}\left(\expec{\norm{\tt-\t_*}_2^4}{}\right)^{3/4}\\
            \leq & \frac{3\atone\sqrt{d_z\bet}}{4}\expec{\norm{\tt-\t_*}_2^4}{}+\frac{\atone\sqrt{d_z\bet}}{4},\numberthis\label{eq:tQttQgS}
        \end{align*}
        where, the first inequality follows by $\lVert\hQt\rVert_2=O(1)$, and $\lVert\Sigma_{ZY}\rVert_2=O(1)$, the second inequality follows by Cauchy-Schwarz inequality, the third inequality follows by Lemma~\ref{lm:gtconvrate} and the fourth inequality follows by Young's inequality.
        \item Similar to \eqref{eq:tQttQgS}, we have,
        \begin{align*}
            &\atone\expec{|(\tt-\t_*)^\top\hQt^2(\tt-\t_*) (\tt-\t_*)^\top\hQt D_t\t_*|}{}\\
             \leq&  \frac{3d_z^{1/2+\varkappa}\atone\sqrt{\bet}}{4}\expec{\norm{\tt-\t_*}_2^4}{}+\frac{d_z^{1/2+\varkappa}\atone\sqrt{\bet}}{4}.\numberthis\label{eq:tQttQDt}
        \end{align*}
        \item Using $\lVert\hQt\rVert_2=O(1)$, Cauchy-Schwarz inequality, \eqref{eq:A1squarebound}, and Young's inequality,
        \begin{align*}
            &\expec{(\tt-\t_*)^\top\hQt^2(\tt-\t_*) A_{1,t}}{}\\
            \leq& \expec{\norm{\tt-\t_*}_2^2 A_{1,t}}{}\\
            \leq& \sqrt{\expec{\norm{\tt-\t_*}_2^4 }{}\expec{ A_{1,t}^2}{}}\\
            \lesssim & d_z^{4\varkappa+\vartheta/2}\atonesqr\lrbrac{1+\expec{\norm{\tt-\t_*}_2^4}{}}.\numberthis\label{eq:A34interact}
        \end{align*}
        \item By Assumption~\ref{aspt: independence}, we have,
        \begin{align*}
            \expec{(\tt-\t_*)^\top\hQt^2(\tt-\t_*) A_{2,t}}{t}=0. \numberthis\label{eq:A35interact}
        \end{align*}
    \end{enumerate}
    Now using Jensen's inequality, and combining \eqref{eq:A31bound}, \eqref{eq:thetaqgtSigmaboud}, \eqref{eq:A33bound}, \eqref{eq:A1squarebound}, and \eqref{eq:A35bound}, we have, 
\begin{align*}
    \expec{A_{3,t}^2}{}\leq& 45\alpha_{t+1}^4\expec{\norm{\gt^\top\xi_{Z_t}\gt(\tt-\t_*)}_2^4}{}+20\atonesqr \expec{((\tt-\t_*)^\top\hQt(\gt-\g_*)^\top \Sigma_{ZY})^2}{}\\
    &+20\atonesqr \expec{((\tt-\t_*)^\top\hQt D_t\t_*)^2}{}+5\expec{A_{1,t}^2}{}+5\expec{A_{2,t}^2}{}\\
    \lesssim & \alpha_{t+1}^4 d_z^{\vartheta_7+8\varkappa} \expec{\norm{\tt-\t^*}_2^4}{}+d_z\atonesqr\bet\left(1+\expec{\norm{\tt-\t_*}_2^4}{}\right)\\
    &+\atonesqr d_z^{1+2\varkappa}\bet\left(1+\expec{\norm{\tt-\t_*}_2^4}{}\right)
    +d_z^{8\varkappa + \vartheta}\alpha_{t+1}^4+\alpha_{t+1}^2d_z^{8\varkappa+\vartheta}(1+\expec{\norm{\tt-\t_*}_2^4}{})\\
     \lesssim &\alpha_{t+1}^2d_z^{8\varkappa+\vartheta}\left(1+\expec{\norm{\tt-\t_*}_2^4}{}\right).\numberthis\label{eq:A3sqrbound}
\end{align*}
Combining \eqref{eq:A31interact}, \eqref{eq:tQttQgS}, \eqref{eq:tQttQDt}, \eqref{eq:A34interact}, and \eqref{eq:A35interact}, we get, 
\begin{align*}
    &\expec{(\tt-\t_*)^\top\hQt^2(\tt-\t_*) A_{3,t}}{}\\
    \lesssim& \atonesqr d_z^{4\varkappa+\vartheta/2}\expec{\norm{\tt-\t_*}_2^4}{}+\frac{3\atone\sqrt{d_z\bet}}{4}\expec{\norm{\tt-\t_*}_2^4}{}+\frac{\atone\sqrt{d_z\bet}}{4}\\
    &+\frac{3d_z^{1/2+\varkappa}\atone\sqrt{\bet}}{4}\expec{\norm{\tt-\t_*}_2^4}{}+\frac{d_z^{1/2+\varkappa}\atone\sqrt{\bet}}{4}
    +d_z^{4\varkappa+\vartheta/2}\atonesqr\lrbrac{1+\expec{\norm{\tt-\t_*}_2^4}{}} \\
    \lesssim& (\atonesqr d_z^{4\varkappa+\vartheta/2}+\atone\sqrt{\betone}d_z^{1/2+\varkappa})(1+\norm{\tt-\t^*}_2^4).\numberthis\label{eq:A3interactbound}
\end{align*}
    Combining \eqref{eq:theta4recintermed}, \eqref{eq:A3sqrbound}, and \eqref{eq:A3interactbound}, we have,
    \begin{align*}
    \expec{\norm{\ttone-\t_*}_2^4}{}
    \lesssim & (1+\atonesqr d_z^{8\varkappa+\vartheta}
    +\atone\sqrt{\betone}d_z^{1/2+\varkappa})\lrbrac{1+\expec{\norm{\tt-\t_*}_2^4}{}}.\numberthis\label{eq:thetaatsqrtbet}
\end{align*}
Now choosing $\at,\bet$ such that $\at\leq d_z^{-4\varkappa-\vartheta/2}$, and $\sum_{t=1}^\infty(\atonesqr 
    +\atone\sqrt{\betone})<\infty$, we get 
\begin{align*}
    \expec{\norm{\tt-\t_*}_2^4}{}\leq M, \numberthis\label{eq:expectheta4bounded}
\end{align*} 
for some constant $0\leq M<\infty$.
\end{proof}

\subsection{Comment on the convergence of ~\eqref{eq:thetaupdatecso}}\label{sec:csoconvergence}
We now discuss the convergence properties of the update sequence~\eqref{eq:thetaupdatecso}, which we refer to as the \emph{conditional stochastic optimization} (CSO) based updates, which we restate below:
\begin{align*}
             \t_{t+1}=\t_t-\atone\g_t^\top Z_t(X_t^\top\tt-Y_t),\quad\quad \g_{t+1}=\g_t-\betone Z_t(Z_t^\top\gt-X_t^\top). 
\end{align*}
Similar to \eqref{eq:thetaupdateexpanded}, for the above updates, we have the following expansion: 
\begin{align*}
    \ttone-\t_*=&\hQt(\tt-\t_*)+\atone(\gt-\g_*)^\top\Sigma_{ZY}+\atone D_t\t_*+\atone\gt^\top\xi_{Z_t}\g_*(\tt-\t_*)\\    &+\atone\gt^\top\xi_{Z_t}\g_*\t_*+\atone\gt^\top\xi_{Z_tY_t}-\atone \gt^\top Z_t\epstt^\top\tt,
\end{align*}
where $\xi_{Z_t}=\Sigma_Z-Z_tZ_t^\top$,  $\xi_{Z_tY_t}=\Sigma_{ZY}-Z_tY_t$, $\hQt\coloneqq\lrbrac{I-\atone\gt^\top\Sigma_Z\g_*}=Q_t+\atone D_t$, and $D_t=(\g_*-\g_t)\Sigma_Z\g_*$.

Recall that the reason for the initial divergence of the updates in~\eqref{eq:thetaupdatecso} are the potential negative eigenvalues of $\gt^\top\Sigma_Z\g_*$. Here we will show that if $\gt^\top\Sigma_Z\g_*$ is positive semi-definite or $\gt$ is close enough to $\g_*$ such that the negative eigenvalues (if any) are not too large in absolute values, then the updates in~\eqref{eq:thetaupdatecso} indeed exhibit the same convergence rate as Algorithm~\ref{alg:one_sample_onlineIV}.
\begin{assumption}\label{as:csocond}
    Let either of the following two conditions be true. For all $t\geq t_0$,
    \begin{enumerate}[noitemsep,leftmargin=0.2in]
        \item $\gt\Sigma_Z\g_*$ is positive semidefinite. 
        \item $\norm{\gt-\g_*}^2\lesssim d_z\bet$.
    \end{enumerate}
\end{assumption}
Note that Condition 1 of Assumption~\ref{as:csocond} is an idealized condition which is difficult to ensure for all $t$ in reality. But of course if this is true, then $\gt\Sigma_Z\g_*$ does not have a negative eigenvalue to cause divergence and the  proof then follows exactly like Lemma~\ref{lm:thetabounded}. 

Hence, we will focus on the more realistic Condition 2 of Assumption~\ref{as:csocond} which holds true almost surely \cite{polyak1992acceleration}. Since we are interested in the asymptotic rate of convergence of CSO updates (due to the requirement of Assumption~\ref{as:csocond}), we will only concentrate on the iterations $t\geq t_0$. In this case, the proof steps are similar to Theorem~\ref{th:mainthetaconv} except for two major differences, that we discuss below.\\

\noindent\textbf{Difference 1: Potential negative definiteness of $\gt^\top \Sigma_Z\g_*$:}  \\

Under Condition 2, $\gt^\top \Sigma_Z\g_*$ can indeed be negative definite. In general, if $\gt^\top \Sigma_Z\g_*$ is negative definite then that is undesirable as we explain Section~\ref{sec:one_sample}. In terms of the proof, we can no longer write $(\tt-\t^*)^\top \hQt^\top\hQt(\tt-\t^*)\leq \norm{\tt-\t_*}^2$ (which was possible to do in \eqref{eq:A12introduced} in the proof of Lemma~\ref{lm:thetabounded}). Subsequently, \eqref{eq:theta4recintermed} breaks down. But we will show that under Condition 2 the negative eigenvalues are not too large in terms of absolute values. Specifically, we can write, 
\begin{align}\label{eq:thetarecursionintermedcso}
\begin{aligned}
    &(\tt-\t^*)^\top \hQt^\top\hQt(\tt-\t^*)\\
    =&(\tt-\t^*)^\top (Q_t^2+\atone Q_t^\top D_t+\atone D_t^\top Q_t+\atonesqr D_t^\top D_t)(\tt-\t^*)\\
    \leq & (1+2\atone\norm{D_t})\norm{\tt-\t_*}^2+\atonesqr\norm{D_t}^2 \norm{\tt-\t_*}^2\\
     \leq  & (1+2\atone\sqrt{d_z\bet})\norm{\tt-\t_*}^2+\atonesqr\norm{D_t}^2 \norm{\tt-\t_*}^2.
\end{aligned}
\end{align}
The term $\atonesqr\norm{D_t}^2 \norm{\tt-\t_*}^2$ is of the order of $A_{3,t}$ defined in \eqref{eq:A3tdef}. Now $\atone\sqrt{d_z\bet}$ is small enough in the sense that we choose the stepsizes such that $\sum_{t=1}^\infty(\atonesqr+\atone\sqrt{\bet})<\infty$. Using this one can now show a similar bound as \eqref{eq:thetaatsqrtbet} and consequently show $\expec{\norm{\tt-\t_*}^4}{}$ is bounded. 

Now let us see what happens in the absence of Condition 2. Here one could use the fact $(1+2\atone\norm{D_t})\lesssim (1+2C_\g\atone d_z^\varkappa)$ which is too big. Recall that we want something at least of the order of $\atone\sqrt{\bet}$ to show that $\tt$ sequence is bounded. One could also try to use the fact that $\expec{\norm{D_t}}{}$ is small by Lemma~\ref{lm:gtconvrate}. But since $D_t$ and $\tt$ are interdependent, one needs to decouple them. One way to do this would be to use Cauchy-Shwarz inequalityas shown below. 
\begin{align*}
    \expec{\norm{D_t}\norm{\tt-\t_*}^2}{}\leq \sqrt{\expec{\norm{D_t}^2}{}\expec{\norm{\tt-\t_*}^4}{}}\lesssim \sqrt{d_z\bet \expec{\norm{\tt-\t_*}^4}{}} .
\end{align*}
But that leads to the presence of $\expec{\norm{\tt-\t_*}^4}{}$ in \eqref{eq:A12introduced} which is potentially problematic due to the fact that on the left-hand side we have $\expec{\norm{\ttone-\t_*}^2}{}$.\\

\noindent\textbf{Difference 2: Presence of additional error term $\atone \gt^\top Z_t\epstt^\top\tt$:} \\

When comparing~\eqref{eq:thetaupdateexpandedcso} with \eqref{eq:thetaupdateexpanded}, yet another crucial difference is the presence of the term $\atone \gt^\top Z_t\epstt^\top\tt$. We will show by the following observations that this error term gets absorbed by other terms already present in \eqref{eq:thetaupdateexpanded} without affecting the convergence rate. Specifically, the following holds.
\begin{enumerate}[leftmargin=0.2in]
    \item Using the independence between $Z$, and $\epstt$, and by Assumption~\ref{aspt: independence}, we have,
    \begin{align*}
        &\EE_{t}[(\hQt(\tt-\t_*)+\atone(\gt-\g_*)^\top\Sigma_{ZY}+\atone D_t\t_*+\atone\gt^\top\xi_{Z_t}\g_*(\tt-\t_*) \\
        &+\atone\gt^\top\xi_{Z_t}\g_*\t_*)^\top \gt^\top Z_t\epstt^\top\tt]=0.
    \end{align*}
    \item We also have that
    \begin{align*}
        &\atonesqr\expec{(\gt^\top\xi_{Z_tY_t})^\top \gt^\top Z_t\epstt^\top\tt}{t} \\
        =&\atonesqr(\gt^\top \Sigma_Z\gt\norm{\t_*}^2+\gt^\top \Sigma_Z\gt\t_*^\top (\tt-\t_*))\\
        \leq & \atonesqr(\gt^\top \Sigma_Z\gt\norm{\t_*}^2+\norm{\gt^\top \Sigma_Z\gt (\tt-\t_*)}^2+\norm{\t_*}^2)
    \end{align*}
    This shows that the above term is of the same order as $A_{1,t}$ and $A_{3,t}$ defined in \eqref{eq:A1tdef}, and \eqref{eq:A3tdef}.
    \item Finally, we have 
    \begin{align*}
        \atonesqr \expec{\norm{\gt^\top Z_t\epstt^\top\tt}^2}{t}\lesssim \atonesqr (\norm{\gt}^2\norm{\tt-\t_*}^2+\norm{\gt}^2\norm{\t_*}^2).
    \end{align*}
    So this term is of the order of $A_{3,t}$ as well. 
\end{enumerate}

Combining the above facts and following similar procedure as the proof of Theorem~\ref{th:mainthetaconv}, one can show that the CSO updates achieve a similar rate under additional Assumption~\ref{as:csocond}.

%% file: main.bbl
\begin{thebibliography}{69}
\providecommand{\natexlab}[1]{#1}
\providecommand{\url}[1]{\texttt{#1}}
\expandafter\ifx\csname urlstyle\endcsname\relax
  \providecommand{\doi}[1]{doi: #1}\else
  \providecommand{\doi}{doi: \begingroup \urlstyle{rm}\Url}\fi

\bibitem[Ahn et~al.(2020)Ahn, Yun, and Sra]{ahn2020sgd}
K.~Ahn, C.~Yun, and S.~Sra.
\newblock {SGD} with shuffling: optimal rates without component convexity and large epoch requirements.
\newblock \emph{Advances in Neural Information Processing Systems}, 33:\penalty0 17526--17535, 2020.

\bibitem[Anastasiou et~al.(2019)Anastasiou, Balasubramanian, and Erdogdu]{anastasiou2019normal}
A.~Anastasiou, K.~Balasubramanian, and M.~A. Erdogdu.
\newblock Normal approximation for stochastic gradient descent via non-asymptotic rates of martingale clt.
\newblock In \emph{Conference on Learning Theory}, pages 115--137. PMLR, 2019.

\bibitem[Angrist and Imbens(1995)]{angrist1995two}
J.~D. Angrist and G.~W. Imbens.
\newblock Two-stage least squares estimation of average causal effects in models with variable treatment intensity.
\newblock \emph{Journal of the American statistical Association}, 90\penalty0 (430):\penalty0 431--442, 1995.

\bibitem[Angrist and Pischke(2009)]{angrist2009mostly}
J.~D. Angrist and J.-S. Pischke.
\newblock \emph{Mostly harmless econometrics: An empiricist's companion}.
\newblock Princeton university press, 2009.

\bibitem[Babii and Florens(2017)]{babii2017completeness}
A.~Babii and J.-P. Florens.
\newblock Is completeness necessary? estimation in nonidentified linear models.
\newblock \emph{arXiv preprint arXiv:1709.03473}, 2017.

\bibitem[Balasubramanian et~al.(2022)Balasubramanian, Ghadimi, and Nguyen]{balasubramanian2022stochastic}
K.~Balasubramanian, S.~Ghadimi, and A.~Nguyen.
\newblock Stochastic multilevel composition optimization algorithms with level-independent convergence rates.
\newblock \emph{SIAM Journal on Optimization}, 32\penalty0 (2):\penalty0 519--544, 2022.

\bibitem[Bennett et~al.(2019)Bennett, Kallus, and Schnabel]{bennett2019deep}
A.~Bennett, N.~Kallus, and T.~Schnabel.
\newblock Deep generalized method of moments for instrumental variable analysis.
\newblock \emph{Advances in neural information processing systems}, 32, 2019.

\bibitem[Bennett et~al.(2023)Bennett, Kallus, Mao, Newey, Syrgkanis, and Uehara]{bennett2023minimax}
A.~Bennett, N.~Kallus, X.~Mao, W.~Newey, V.~Syrgkanis, and M.~Uehara.
\newblock Minimax instrumental variable regression and $ l\_2 $ convergence guarantees without identification or closedness.
\newblock \emph{arXiv preprint arXiv:2302.05404}, 2023.

\bibitem[Carrasco et~al.(2007)Carrasco, Florens, and Renault]{carrasco2007linear}
M.~Carrasco, J.-P. Florens, and E.~Renault.
\newblock Linear inverse problems in structural econometrics estimation based on spectral decomposition and regularization.
\newblock \emph{Handbook of econometrics}, 6:\penalty0 5633--5751, 2007.

\bibitem[Centorrino and Florens(2021)]{centorrino2021nonparametric}
S.~Centorrino and J.-P. Florens.
\newblock Nonparametric instrumental variable estimation of binary response models with continuous endogenous regressors.
\newblock \emph{Econometrics and Statistics}, 17:\penalty0 35--63, 2021.

\bibitem[Chen et~al.(2021)Chen, Sun, and Yin]{chen2021solving}
T.~Chen, Y.~Sun, and W.~Yin.
\newblock Solving stochastic compositional optimization is nearly as easy as solving stochastic optimization.
\newblock \emph{IEEE Transactions on Signal Processing}, 69:\penalty0 4937--4948, 2021.

\bibitem[Chen and Pouzo(2012)]{chen2012estimation}
X.~Chen and D.~Pouzo.
\newblock Estimation of nonparametric conditional moment models with possibly nonsmooth generalized residuals.
\newblock \emph{Econometrica}, 80\penalty0 (1):\penalty0 277--321, 2012.

\bibitem[Chen and Reiss(2011)]{chen2011rate}
X.~Chen and M.~Reiss.
\newblock On rate optimality for ill-posed inverse problems in econometrics.
\newblock \emph{Econometric Theory}, 27\penalty0 (3):\penalty0 497--521, 2011.

\bibitem[Chen et~al.(2020)Chen, Lee, Tong, and Zhang]{chen2020statistical}
X.~Chen, J.~D. Lee, X.~T. Tong, and Y.~Zhang.
\newblock Statistical inference for model parameters in stochastic gradient descent.
\newblock \emph{Annals of Statistics}, 48\penalty0 (1):\penalty0 251--273, 2020.

\bibitem[Chen et~al.(2023)Chen, Lee, Liao, Seo, Shin, and Song]{chen2023sgmm}
X.~Chen, S.~Lee, Y.~Liao, M.~H. Seo, Y.~Shin, and M.~Song.
\newblock {SGMM: Stochastic approximation to generalized method of moments}.
\newblock \emph{arXiv preprint arXiv:2308.13564}, 2023.

\bibitem[Cui et~al.(2023)Cui, Pu, Shi, Miao, and Tchetgen~Tchetgen]{cui2023semiparametric}
Y.~Cui, H.~Pu, X.~Shi, W.~Miao, and E.~Tchetgen~Tchetgen.
\newblock Semiparametric proximal causal inference.
\newblock \emph{Journal of the American Statistical Association}, pages 1--12, 2023.

\bibitem[Dai et~al.(2017)Dai, He, Pan, Boots, and Song]{dai2017learning}
B.~Dai, N.~He, Y.~Pan, B.~Boots, and L.~Song.
\newblock Learning from conditional distributions via dual embeddings.
\newblock In \emph{Artificial Intelligence and Statistics}, pages 1458--1467. PMLR, 2017.

\bibitem[Dalal et~al.(2018)Dalal, Thoppe, Sz{\"o}r{\'e}nyi, and Mannor]{dalal2018finite}
G.~Dalal, G.~Thoppe, B.~Sz{\"o}r{\'e}nyi, and S.~Mannor.
\newblock Finite sample analysis of two-timescale stochastic approximation with applications to reinforcement learning.
\newblock In \emph{Conference On Learning Theory}, pages 1199--1233. PMLR, 2018.

\bibitem[Darolles et~al.(2011)Darolles, Fan, Florens, and Renault]{darolles2011nonparametric}
S.~Darolles, Y.~Fan, J.-P. Florens, and E.~Renault.
\newblock Nonparametric instrumental regression.
\newblock \emph{Econometrica}, 79\penalty0 (5):\penalty0 1541--1565, 2011.

\bibitem[Della~Vecchia and Basu(2023)]{della2023onlinev1}
R.~Della~Vecchia and D.~Basu.
\newblock Online instrumental variable regression: Regret analysis and bandit feedback.
\newblock \emph{arXiv preprint arXiv:2302.09357v1}, 2023.

\bibitem[Della~Vecchia and Basu(2024)]{della2023online}
R.~Della~Vecchia and D.~Basu.
\newblock Stochastic online instrumental variable regression: Regrets for endogeneity and bandit feedback.
\newblock \emph{arXiv preprint arXiv:2302.09357v3}, 2024.

\bibitem[Dikkala et~al.(2020)Dikkala, Lewis, Mackey, and Syrgkanis]{dikkala2020minimax}
N.~Dikkala, G.~Lewis, L.~Mackey, and V.~Syrgkanis.
\newblock Minimax estimation of conditional moment models.
\newblock \emph{Advances in Neural Information Processing Systems}, 33:\penalty0 12248--12262, 2020.

\bibitem[Doan and Romberg(2020)]{doan2020finite}
T.~Doan and J.~Romberg.
\newblock Finite-time performance of distributed two-time-scale stochastic approximation.
\newblock In \emph{Learning for Dynamics and Control}, pages 26--36. PMLR, 2020.

\bibitem[Doan(2022)]{doan2022nonlinear}
T.~T. Doan.
\newblock Nonlinear two-time-scale stochastic approximation convergence and finite-time performance.
\newblock \emph{IEEE Transactions on Automatic Control}, 2022.

\bibitem[Duchi et~al.(2012)Duchi, Agarwal, Johansson, and Jordan]{duchi2012ergodic}
J.~C. Duchi, A.~Agarwal, M.~Johansson, and M.~I. Jordan.
\newblock Ergodic mirror descent.
\newblock \emph{SIAM Journal on Optimization}, 22\penalty0 (4):\penalty0 1549--1578, 2012.

\bibitem[Ermoliev and Norkin(2013)]{ermoliev2013sample}
Y.~M. Ermoliev and V.~I. Norkin.
\newblock Sample average approximation method for compound stochastic optimization problems.
\newblock \emph{SIAM Journal on Optimization}, 23\penalty0 (4):\penalty0 2231--2263, 2013.

\bibitem[Even(2023)]{even2023stochastic}
M.~Even.
\newblock Stochastic gradient descent under {M}arkovian sampling schemes.
\newblock In \emph{International Conference on Machine Learning}, pages 9412--9439. PMLR, 2023.

\bibitem[Ghadimi and Lan(2013)]{ghadimi2013stochastic}
S.~Ghadimi and G.~Lan.
\newblock Stochastic first-and zeroth-order methods for nonconvex stochastic programming.
\newblock \emph{SIAM journal on optimization}, 23\penalty0 (4):\penalty0 2341--2368, 2013.

\bibitem[Ghadimi et~al.(2020)Ghadimi, Ruszczynski, and Wang]{ghadimi2020single}
S.~Ghadimi, A.~Ruszczynski, and M.~Wang.
\newblock A single timescale stochastic approximation method for nested stochastic optimization.
\newblock \emph{SIAM Journal on Optimization}, 30\penalty0 (1):\penalty0 960--979, 2020.

\bibitem[Gurbuzbalaban et~al.(2019)Gurbuzbalaban, Ozdaglar, and Parrilo]{gurbuzbalaban2019convergence}
M.~Gurbuzbalaban, A.~Ozdaglar, and P.~A. Parrilo.
\newblock Convergence rate of incremental gradient and incremental newton methods.
\newblock \emph{SIAM Journal on Optimization}, 29\penalty0 (4):\penalty0 2542--2565, 2019.

\bibitem[Hall and Horowitz(2005)]{hall2005nonparametric}
P.~Hall and J.~L. Horowitz.
\newblock Nonparametric methods for inference in the presence of instrumental variables.
\newblock \emph{Annals of statistics}, 33\penalty0 (6):\penalty0 2904--2929, 2005.

\bibitem[Haochen and Sra(2019)]{haochen2019random}
J.~Haochen and S.~Sra.
\newblock Random shuffling beats {SGD} after finite epochs.
\newblock In \emph{International Conference on Machine Learning}, pages 2624--2633. PMLR, 2019.

\bibitem[Hartford et~al.(2017)Hartford, Lewis, Leyton-Brown, and Taddy]{hartford2017deep}
J.~Hartford, G.~Lewis, K.~Leyton-Brown, and M.~Taddy.
\newblock Deep iv: A flexible approach for counterfactual prediction.
\newblock In \emph{Proceedings of the 34th International Conference on Machine Learning-Volume 70}, pages 1414--1423. JMLR, 2017.

\bibitem[Hu et~al.(2020{\natexlab{a}})Hu, Chen, and He]{hu2020sample}
Y.~Hu, X.~Chen, and N.~He.
\newblock Sample complexity of sample average approximation for conditional stochastic optimization.
\newblock \emph{SIAM Journal on Optimization}, 30\penalty0 (3):\penalty0 2103--2133, 2020{\natexlab{a}}.

\bibitem[Hu et~al.(2020{\natexlab{b}})Hu, Zhang, Chen, and He]{hu2020biased}
Y.~Hu, S.~Zhang, X.~Chen, and N.~He.
\newblock Biased stochastic first-order methods for conditional stochastic optimization and applications in meta learning.
\newblock \emph{Advances in Neural Information Processing Systems}, 33:\penalty0 2759--2770, 2020{\natexlab{b}}.

\bibitem[Hu et~al.(2021)Hu, Chen, and He]{hu2021bias}
Y.~Hu, X.~Chen, and N.~He.
\newblock On the bias-variance-cost tradeoff of stochastic optimization.
\newblock \emph{Advances in Neural Information Processing Systems}, 34:\penalty0 22119--22131, 2021.

\bibitem[Hu et~al.(2024)Hu, Wang, Xie, Krause, and Kuhn]{hu2024contextual}
Y.~Hu, J.~Wang, Y.~Xie, A.~Krause, and D.~Kuhn.
\newblock Contextual stochastic bilevel optimization.
\newblock \emph{Advances in Neural Information Processing Systems}, 36, 2024.

\bibitem[Khaled and Richt{\'a}rik(2020)]{khaled2020better}
A.~Khaled and P.~Richt{\'a}rik.
\newblock Better theory for {SGD} in the nonconvex world.
\newblock \emph{arXiv preprint arXiv:2002.03329}, 2020.

\bibitem[Lan(2020)]{lan2020first}
G.~Lan.
\newblock \emph{First-order and stochastic optimization methods for machine learning}, volume~1.
\newblock Springer, 2020.

\bibitem[Lewis and Syrgkanis(2018)]{lewis2018adversarial}
G.~Lewis and V.~Syrgkanis.
\newblock Adversarial generalized method of moments.
\newblock \emph{arXiv preprint arXiv:1803.07164}, 2018.

\bibitem[Liao et~al.(2020)Liao, Chen, Yang, Dai, Kolar, and Wang]{liao2020provably}
L.~Liao, Y.-L. Chen, Z.~Yang, B.~Dai, M.~Kolar, and Z.~Wang.
\newblock Provably efficient neural estimation of structural equation models: An adversarial approach.
\newblock \emph{Advances in Neural Information Processing Systems}, 33:\penalty0 8947--8958, 2020.

\bibitem[Maei et~al.(2009)Maei, Szepesvari, Bhatnagar, Precup, Silver, and Sutton]{maei2009convergent}
H.~Maei, C.~Szepesvari, S.~Bhatnagar, D.~Precup, D.~Silver, and R.~S. Sutton.
\newblock Convergent temporal-difference learning with arbitrary smooth function approximation.
\newblock \emph{Advances in neural information processing systems}, 22, 2009.

\bibitem[Mastouri et~al.(2021)Mastouri, Zhu, Gultchin, Korba, Silva, Kusner, Gretton, and Muandet]{mastouri2021proximal}
A.~Mastouri, Y.~Zhu, L.~Gultchin, A.~Korba, R.~Silva, M.~Kusner, A.~Gretton, and K.~Muandet.
\newblock Proximal causal learning with kernels: Two-stage estimation and moment restriction.
\newblock In \emph{International conference on machine learning}, pages 7512--7523. PMLR, 2021.

\bibitem[Mokkadem and Pelletier(2006)]{mokkadem2006convergence}
A.~Mokkadem and M.~Pelletier.
\newblock Convergence rate and averaging of nonlinear two-time-scale stochastic approximation algorithms.
\newblock \emph{Annals of Applied Probability}, 16\penalty0 (3):\penalty0 1671--1702, 2006.

\bibitem[Muandet et~al.(2020)Muandet, Mehrjou, Lee, and Raj]{muandet2020dual}
K.~Muandet, A.~Mehrjou, S.~K. Lee, and A.~Raj.
\newblock Dual instrumental variable regression.
\newblock \emph{Advances in Neural Information Processing Systems}, 33:\penalty0 2710--2721, 2020.

\bibitem[Nagaraj et~al.(2019)Nagaraj, Jain, and Netrapalli]{nagaraj2019sgd}
D.~Nagaraj, P.~Jain, and P.~Netrapalli.
\newblock {SGD} without replacement: Sharper rates for general smooth convex functions.
\newblock In \emph{International Conference on Machine Learning}, pages 4703--4711. PMLR, 2019.

\bibitem[Nesterov(2013)]{nesterov2013introductory}
Y.~Nesterov.
\newblock \emph{Introductory lectures on convex optimization: A basic course}, volume~87.
\newblock Springer Science \& Business Media, 2013.

\bibitem[Newey and Powell(2003)]{newey2003instrumental}
W.~K. Newey and J.~L. Powell.
\newblock Instrumental variable estimation of nonparametric models.
\newblock \emph{Econometrica}, 71\penalty0 (5):\penalty0 1565--1578, 2003.

\bibitem[Papadimitriou(2003)]{papadimitriou2003computational}
C.~H. Papadimitriou.
\newblock Computational complexity.
\newblock In \emph{Encyclopedia of computer science}, pages 260--265. 2003.

\bibitem[Peixoto et~al.(2024)Peixoto, Saporito, and Fonseca]{peixoto2024nonparametric}
C.~Peixoto, Y.~Saporito, and Y.~Fonseca.
\newblock Nonparametric instrumental variable regression through stochastic approximate gradients.
\newblock \emph{arXiv preprint arXiv:2402.05639}, 2024.

\bibitem[Polyak and Juditsky(1992)]{polyak1992acceleration}
B.~T. Polyak and A.~B. Juditsky.
\newblock Acceleration of stochastic approximation by averaging.
\newblock \emph{SIAM journal on control and optimization}, 30\penalty0 (4):\penalty0 838--855, 1992.

\bibitem[Rajput et~al.(2020)Rajput, Gupta, and Papailiopoulos]{rajput2020closing}
S.~Rajput, A.~Gupta, and D.~Papailiopoulos.
\newblock Closing the convergence gap of {SGD} without replacement.
\newblock In \emph{International Conference on Machine Learning}, pages 7964--7973. PMLR, 2020.

\bibitem[Reiers{\o}l(1945)]{reiersol1945confluence}
O.~Reiers{\o}l.
\newblock \emph{Confluence analysis by means of instrumental sets of variables}.
\newblock PhD thesis, Almqvist \& Wiksell, 1945.

\bibitem[Roy et~al.(2022)Roy, Balasubramanian, and Ghadimi]{roy2022constrained}
A.~Roy, K.~Balasubramanian, and S.~Ghadimi.
\newblock Constrained stochastic nonconvex optimization with state-dependent {M}arkov data.
\newblock \emph{Advances in Neural Information Processing Systems}, 35:\penalty0 23256--23270, 2022.

\bibitem[Ruszczynski(2021)]{ruszczynski2021stochastic}
A.~Ruszczynski.
\newblock A stochastic subgradient method for nonsmooth nonconvex multilevel composition optimization.
\newblock \emph{SIAM Journal on Control and Optimization}, 59\penalty0 (3):\penalty0 2301--2320, 2021.

\bibitem[Shalev-Shwartz et~al.(2012)]{shalev2012online}
S.~Shalev-Shwartz et~al.
\newblock Online learning and online convex optimization.
\newblock \emph{Foundations and Trends{\textregistered} in Machine Learning}, 4\penalty0 (2):\penalty0 107--194, 2012.

\bibitem[Shao and Zhang(2022)]{shao2022berry}
Q.-M. Shao and Z.-S. Zhang.
\newblock Berry--esseen bounds for multivariate nonlinear statistics with applications to m-estimators and stochastic gradient descent algorithms.
\newblock \emph{Bernoulli}, 28\penalty0 (3):\penalty0 1548--1576, 2022.

\bibitem[Singh et~al.(2019)Singh, Sahani, and Gretton]{singh2019kernel}
R.~Singh, M.~Sahani, and A.~Gretton.
\newblock Kernel instrumental variable regression.
\newblock \emph{Advances in Neural Information Processing Systems}, 32, 2019.

\bibitem[Sun et~al.(2018)Sun, Sun, and Yin]{sun2018markov}
T.~Sun, Y.~Sun, and W.~Yin.
\newblock On {M}arkov chain gradient descent.
\newblock \emph{Advances in neural information processing systems}, 31, 2018.

\bibitem[Tseng(1998)]{tseng1998incremental}
P.~Tseng.
\newblock An incremental gradient (-projection) method with momentum term and adaptive stepsize rule.
\newblock \emph{SIAM Journal on Optimization}, 8\penalty0 (2):\penalty0 506--531, 1998.

\bibitem[Venkatraman et~al.(2016)Venkatraman, Sun, Hebert, Bagnell, and Boots]{venkatraman2016online}
A.~Venkatraman, W.~Sun, M.~Hebert, J.~Bagnell, and B.~Boots.
\newblock Online instrumental variable regression with applications to online linear system identification.
\newblock In \emph{Proceedings of the AAAI Conference on Artificial Intelligence}, volume~30, 2016.

\bibitem[Wang et~al.(2017)Wang, Fang, and Liu]{wang2017stochastic}
M.~Wang, E.~X. Fang, and B.~Liu.
\newblock Stochastic compositional gradient descent: Algorithms for minimizing compositions of expected-value functions.
\newblock \emph{Mathematical Programming}, 161\penalty0 (1-2):\penalty0 419--449, 2017.

\bibitem[Wang et~al.(2021)Wang, Zou, and Zhou]{wang2021non}
Y.~Wang, S.~Zou, and Y.~Zhou.
\newblock Non-asymptotic analysis for two time-scale tdc with general smooth function approximation.
\newblock \emph{Advances in Neural Information Processing Systems}, 34:\penalty0 9747--9758, 2021.

\bibitem[Wright(1928)]{wright1928tariff}
P.~G. Wright.
\newblock \emph{The tariff on animal and vegetable oils}.
\newblock Number~26. Macmillan, 1928.

\bibitem[Xu et~al.(2021)Xu, Chen, Srinivasan, de~Freitas, Doucet, and Gretton]{xu2021learning}
L.~Xu, Y.~Chen, S.~Srinivasan, N.~de~Freitas, A.~Doucet, and A.~Gretton.
\newblock Learning deep features in instrumental variable regression.
\newblock In \emph{International Conference on Learning Representations}, 2021.
\newblock URL \url{https://openreview.net/forum?id=sy4Kg_ZQmS7}.

\bibitem[Xu and Liang(2021)]{xu2021sample}
T.~Xu and Y.~Liang.
\newblock Sample complexity bounds for two timescale value-based reinforcement learning algorithms.
\newblock In \emph{International Conference on Artificial Intelligence and Statistics}, pages 811--819. PMLR, 2021.

\bibitem[Zhang and Xiao(2021)]{zhang2021multilevel}
J.~Zhang and L.~Xiao.
\newblock Multilevel composite stochastic optimization via nested variance reduction.
\newblock \emph{SIAM Journal on Optimization}, 31\penalty0 (2):\penalty0 1131--1157, 2021.

\bibitem[Zhu et~al.(2023)Zhu, Chen, and Wu]{zhu2023online}
W.~Zhu, X.~Chen, and W.~B. Wu.
\newblock Online covariance matrix estimation in stochastic gradient descent.
\newblock \emph{Journal of the American Statistical Association}, 118\penalty0 (541):\penalty0 393--404, 2023.

\bibitem[Zhu et~al.(2022)Zhu, Gultchin, Gretton, Kusner, and Silva]{zhu2022causal}
Y.~Zhu, L.~Gultchin, A.~Gretton, M.~J. Kusner, and R.~Silva.
\newblock Causal inference with treatment measurement error: a nonparametric instrumental variable approach.
\newblock In \emph{Uncertainty in Artificial Intelligence}, pages 2414--2424. PMLR, 2022.

\end{thebibliography}
